%% file: main.tex
\pdfoutput=1
\documentclass[11pt]{article}
\usepackage[font=libertinus, citestyle=authoryear]{kurbanlab}
\input{affiliations}

\newcolumntype{Y}[1]{>{\centering\arraybackslash}p{#1}}
\usetikzlibrary{arrows.meta,positioning,fit,calc}
\algtext*{EndIf}\algtext*{EndFor}\algtext*{EndWhile}\algtext*{EndProcedure}\algtext*{EndFunction}
\algrenewcommand\algorithmicrequire{\textbf{Input:}}
\algrenewcommand\algorithmicensure{\textbf{Output:}}
\algrenewcommand\alglinenumber[1]{\scriptsize #1}
\algrenewcommand{\algorithmiccomment}[1]{\hfill{\color{gray}$\triangleright$~#1}}
\DeclareMathSizes{7.4}{7.4}{6}{5}
\definecolor{cInk}{RGB}{28,28,30}
\definecolor{cCert}{RGB}{13,102,143}
\definecolor{cCertF}{RGB}{228,240,247}
\definecolor{cPend}{RGB}{112,112,116}
\definecolor{cPendF}{RGB}{243,243,244}
\definecolor{cExcl}{RGB}{176,84,12}
\definecolor{cExclF}{RGB}{252,240,229}
\definecolor{cRule}{RGB}{196,196,200}

\newlength{\tinset}
\newcommand{\thead}{\toprule}
\newcommand{\theadn}[2]{\toprule}
\newcommand{\theadend}{\midrule}
\newcommand{\tgroupsep}{\midrule}
\newcommand{\tgroupc}[5]{\multicolumn{#3}{@{}l@{}}{\emph{#4}\hspace{0.9em}\textcolor{cPend}{#5}}\\}
\newcommand{\tgroup}[3]{\tgroupc{}{}{#1}{#2}{#3}}
\newcommand{\tgroupcert}[3]{\tgroupc{}{}{#1}{#2}{#3}}
\newcommand{\tgroupexcl}[3]{\tgroupc{}{}{#1}{#2}{#3}}
\newcommand{\tunder}[1]{\midrule}
\newcommand{\tunderend}{}
\newcommand{\thairx}[1]{\noalign{\vskip#1}}
\newcommand{\thair}{\thairx{2pt}}
\newcommand{\tgreyrule}[1]{\noalign{\vskip#1{\color{cRule}\hrule height0.4pt}\vskip#1}}
\newcommand{\thd}[1]{\textbf{#1}}
\newcommand{\thdsub}[1]{\textcolor{cPend}{#1}}
\newcommand{\tmut}[1]{\textcolor{cPend}{#1}}
\newcommand{\tkey}[2]{\textcolor{#1}{#2}}
\newcommand{\tpartial}[1]{#1}

\newcommand{\bx}{\mathbf{x}}
\newcommand{\bz}{\mathbf{z}}
\newcommand{\bs}{\mathbf{s}}
\newcommand{\bd}{\mathbf{D}}
\newcommand{\Ac}{\mathcal{A}}
\newcommand{\Ex}{\mathbb{E}}
\newcommand{\Lc}{\mathcal{L}}
\newcommand{\TV}{\mathrm{TV}}
\newcommand{\IAS}{\textsc{ias}}
\newcommand{\AF}{\textsc{AliasForge}}
\newcommand{\ROUTE}{\textsc{ias-Route}}

\title{Certified Interface Aliases:\texorpdfstring{\\}{ }Exact Collisions in Vision--Language\texorpdfstring{\\}{ }Preprocessing, and When They Exist}
\RunningTitle{Certified interface aliases}

\Author{Mert Onur Cakiroglu}{iub}
\Author{Elham Buxton}{uis}
\Author{Mehmet Dalkilic}{iub}
\Author[corresponding=hkurban@hbku.edu.qa, orcid=0000-0003-3142-2866]{Hasan Kurban}{hbku}

\CodeURL{https://github.com/KurbanIntelligenceLab/aliasforge}
\Venue{Under review}

\begin{document}
\maketitle

\begin{abstract} 
Vision--language verifiers and routers must distinguish errors repairable by more reasoning from those caused by visual evidence never reaching the language model. This distinction lacks ground truth because annotators see full-resolution images while models receive preprocessed tensors. We introduce \AF{} to create cases where the relevant fact is provably absent from the interface. Fixed-point resampling makes the pre-rounding resize an exact integer linear map that can send nonzero integer perturbations to zero. Hiding a label-flipping perturbation there produces images with opposite step-correctness labels but bit-identical interface states. Every verifier therefore has the same output law on both members, giving pair-balanced accuracy exactly one half and zero gain from language-side repair. We prove that every fixed-point downscaler has such null vectors and bound their smallest size at the most common ratios, which rules out an 8-bit fit whenever the bound exceeds $255$. From the resize configuration alone, a lattice criterion supplies realizable collisions and certifies their absence within the specified construction family. It resolves all $20$ screened configurations, $17$ as constructible and $3$ as non-constructible. We construct certified pairs across three architectures and certify four additional processors, with zero decision-logit gap on all $18$ scored pairs and none of the $18$ controls. The pairs also screen routers that waste computation on re-attention or further reasoning. On natural items, per-item routing headroom exists, but no tested interface-only router improves over stopping. Our fiber ceiling bounds the headroom recoverable from the interface.
Code: \url{https://github.com/KurbanIntelligenceLab/aliasforge}.
\end{abstract}

\printkeywords

\section{Introduction}\label{sec:intro}

When a vision--language model fails on a visual task, diagnosing the error requires knowing whether the visual evidence ever reached the language model. Nothing the model emits reveals this. Yet practitioners are forced to act without this knowledge. They train perception-aware reward models \citep{perceval,pearl,gpro}, decide whether a downsampled view suffices \citep{visionthink}, and route test-time compute accordingly \citep{damani,avis}. None of these procedures can verify if the visual evidence survived preprocessing, and guessing incorrectly is not free. Consider when a verifier model, tasked with judging a reasoning step against the image, is instructed to re-examine an input whose decisive detail never survived preprocessing. In this scenario, the system wastes compute on a channel that carries no useful signal. Conversely, re-encoding at a higher resolution when the detail was already present wastes compute on an unnecessary intervention. Similarly, a reward model trained on perception labels assigned by a judge model simply inherits the judge's own perceptual blind spots. Ultimately, these mechanisms are making decisions based on missing information. The resulting errors compound quietly because, from the outside, a model that could not see the evidence behaves exactly like one that did not use it \citep{blindtest,compensate}.

This gap exists because human evaluators establish ground truth from the full-resolution image, whereas the model receives only a tensor. Two inputs whose tensors differ at all, however slightly, can be told apart by some verifier (Proposition~\ref{prop:sharp}), so no threshold on tensor similarity can certify that the visual evidence failed to reach the model. Our certificate comes from one observation. An integer resize has an integer null space, and we hide a label-flipping perturbation in it.

\begin{figure}[t] \centering
\includegraphics[width=0.95\textwidth,trim=0 3pt 0 10pt,clip]{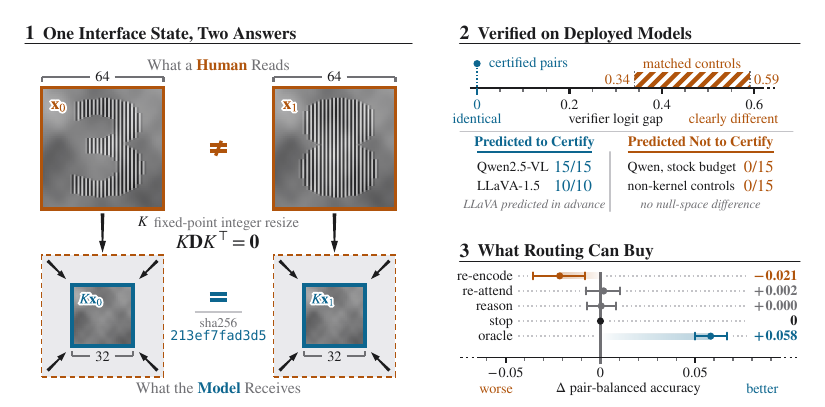} \caption{\textbf{Certified interface aliases and routing.} (1)~Images with different correct answers reach the model as bit-identical tensors because the deployed resize $K$ maps their difference $\bd$ to zero (Theorem~\ref{thm:cert}). (2)~The criterion predicts which deployed pipelines admit certified pairs, while controls outside the resize null space and identity-resize controls do not certify (Proposition~\ref{prop:construct}). (3)~No tested action improves on stopping on natural items despite per-item oracle headroom; error bars are $95\%$ intervals clustered by record, the source problem.}
\label{fig:teaser}
\end{figure}

In most deployed preprocessing pipelines the input image is resized by a resampler to the size the model expects. A deployed resampler's filter weights are fixed-point integers, so its downscaling resize is an exact linear map over the integers, and some integer differences between images map to zero. We call two such images a \emph{certified pair}\footnote{This notion differs from certified adversarial robustness, which proves that predictions remain fixed under bounded perturbations. Our certificate proves that inputs with different correct answers reach the model identically.} when they have opposite ground-truth labels under one shared question. Because their difference maps to zero, they reach the model as bit-identical tensors, so it provably cannot tell them apart (Figure~\ref{fig:teaser}). Our constructor \AF{} hides a label-flipping perturbation in one of those differences (Section~\ref{sec:af}). On such a pair every failure is a perception failure, and nothing the model does can help unless it changes the image. From the outside, chance accuracy on a certified pair looks the same as chance accuracy on two images that are merely similar. Only on a certified pair is that chance provable rather than observed. Constructed pairs test verifiers and routers on known perceptual failures, and they do not by themselves diagnose a natural error. The mechanism descends from image-scaling attacks \citep{scalingattack,advpreproc}, repurposed from evasion to measurement.

We build \AF{}, the constructor that produces certified pairs, validate it on three architectures, and audit $20$ declared configurations with its criterion. We then ask the question that motivated the work: \emph{Can a score read from the preprocessed tensor alone decide when an input deserves more compute?} No score we test can, and the fiber ceiling (Proposition~\ref{prop:fiber}) bounds what any such score could recover. Three contributions follow.

\begin{itemize}[leftmargin=1.4em,itemsep=1pt,topsep=2pt,parsep=0pt]
\item \textbf{A certified measurement instrument.} We propose and implement \AF{}, which constructs certified pairs whose members have opposite correct answers but bit-identical preprocessed states. We prove that no verifier or language-side repair can distinguish their members (Theorem~\ref{thm:cert}), and empirically validate the construction across three architectures and four additional processors (Tables~\ref{tab:e0} and~\ref{tab:breadth}).
\item \textbf{A constructibility criterion from preprocessing alone.} We derive a criterion that uses only the deployed resize operation to test whether our construction can produce a certified pair, without model weights or queries (Theorem~\ref{thm:poly}, Proposition~\ref{prop:construct}). Across all $20$ declared configurations, it either produces a pair or certifies that our construction cannot produce one. It also correctly predicts constructibility on a second architecture before any pairs are generated.
\item \textbf{A screen and ceiling for interface-only routing.} We use certified pairs to test whether a router spends compute on re-attention or further reasoning when those actions provably cannot help (Corollary~\ref{cor:screen}). We also derive the fiber ceiling, which bounds the gain of any router that selects actions using only the preprocessed image (Proposition~\ref{prop:fiber}). On natural benchmark items, the best action varies across items, but none of the interface-only routers we test improves on always stopping (Tables~\ref{tab:main} and~\ref{tab:routers}).
\end{itemize}

\section{Related work}

\label{sec:related}

\paragraph{Approximate and exact collisions} Prior work identifies visually distinct images with similar internal representations. The Multimodal Visual Patterns (MMVP) benchmark selects ``CLIP-blind'' image pairs by thresholding embedding similarity \citep{mmvp}. Other work uses optimization to make images approach a target embedding \citep{equivstructures,unaligning}, while connector reconstruction loss has been used to predict failures in grounded question answering \citep{lostinembeddings}. These approaches identify approximate similarity or predict failure, but do not certify exact equality of the deployed interface (Proposition~\ref{prop:sharp}). Text preprocessing shows the opposite behavior. Visually identical Unicode strings can produce different token sequences \citep{badchars}, while our visually different images produce the same preprocessed representation. The closest constructive precedent is image-scaling attacks, which create inputs whose downscaled content differs from their full-resolution appearance \citep{scalingattack}, a behavior arising from the interaction between downsampling and convolution \citep{advpreproc}. These attacks have been demonstrated against deployed multimodal systems \citep{chameleon}, and the preprocessing configuration can be recovered through black-box queries \citep{patchsize}. Defenses developed for these attacks aim to prevent attacker-chosen content from appearing after downscaling \citep{advpreproc}, whereas our aliases (Definition~\ref{as:hash}) leave the downscaled output unchanged while altering the correct answer in the full-resolution image. Randomized resizing can move bicubic preprocessing away from an integer downscaling ratio \citep{randomization}, but changing the resampling kernel does not necessarily eliminate aliases. Area scaling has been reported to resist image-scaling attacks \citep{advpreproc}. However, at the integer resize ratios evaluated in this study, it is equivalent to Pillow's box filter, which admits exact collisions at every ratio we successfully reconstructed.

\paragraph{Representation probing} A representation probe trains a decoder to predict a visual fact from the model's internal representation. Successful decoding provides evidence that the fact is present, but failure does not prove absence because the decoder itself may be inadequate. Control tasks help detect probes that succeed by fitting arbitrary labels \citep{controltasks}. Recent theory characterizes information loss through inputs that share a representation \citep{bayessuff,fibercriterion}, but does not identify such collisions in a particular deployed system. Our certified pairs provide such a witness for a deployed interface.

\paragraph{Perception labels in verifiers and routers} Prior work trains verifiers to identify whether a reasoning error is perceptual \citep{perceval,pearl,vlprm,badseeing}, evaluates this ability at the reasoning-step level \citep{prmbench}, and uses similar signals to decide when additional computation is needed \citep{damani,failtosee,gpro,avis,imagine}. Among these methods, \citet{visionthink} is closest to our setting because it routes compute using an estimate of whether the relevant visual evidence is available. \citet{blindtest} study a complementary failure in which a visual fact is represented but not used, while other studies infer limited visual access from model behavior \citep{seeingnotbelieving,vlmneedwords,visualaccess}. Correctly identifying the source of an error matters because one audit assigns $86.9\%$ of a strong model's visual-mathematics errors to perception using a judge model that cannot verify what visual information reached the model \citep{seeing2thinking}. Our certified pairs establish a known case in which the distinguishing fact is absent from the interface and provide a one-forward-pass test of a router's decision without observing intervention outcomes (Corollary~\ref{cor:screen}). Prior work finds that longer reasoning can disperse visual attention \citep{deeperthought,ttsvlm}, prompted self-reflection can fail to revisit the image \citep{visualswap}, and chains of thought can propagate upstream perceptual errors \citep{disentangle}. Our routing evaluation complements these findings by keeping the verifier frozen, separating re-attention \citep{lookagain} from re-encoding \citep{sieve}, and measuring regret against a per-item oracle. 


\section{Setup}\label{sec:setup}

A verifier $f(\bx,Q,R)\in[0,1]$ assigns a probability that candidate reasoning step $R$ is correct for image $\bx$ and question $Q$. The image reaches the language model only through the \emph{visual interface} $\bz=\Phi(\bx)$, which contains every image-dependent quantity on which the language model conditions. Formally, the verifier is a Markov kernel from $(\bz,Q,R)$ to a verdict in $\{0,1\}$, and $f(\bx,Q,R)$ is the probability that it accepts $R$. This covers deterministic scoring and sampled decoding, and its randomness $\omega$ is drawn independently of $(\bx,\bz)$. The deployed processor computes this interface in stages, so $\Phi=\Phi_n\circ\cdots\circ\Phi_1$, and we write $\bs_k(\bx)=\Phi_k(\cdots\Phi_1(\bx))$ for the state after stage $k$, with $\bs_n(\bx)=\Phi(\bx)$. Write $y(\bx,Q,R)\in\{0,1\}$ for the ground-truth step-correctness label, determined at full resolution. A \emph{pair} is $(\bx_0,\bx_1,Q,R)$ with $y(\bx_0,Q,R)=1$ and $y(\bx_1,Q,R)=0$. For a pair $i=(\bx_{i,0},\bx_{i,1},Q_i,R_i)$ and member $b\in\{0,1\}$, write $p_{i,b}=f(\bx_{i,b},Q_i,R_i)$. The \emph{pair-balanced accuracy} is $\mathrm{BA}=\Ex_i[\tfrac12(p_{i,0}+1-p_{i,1})]$.

To study routing, we use the action set $\Ac=\{\textsf{stop},\textsf{att},\textsf{rea},\textsf{enc}\}$. Here, $\textsf{stop}$ accepts the current verdict, $\textsf{att}$ (re-attend) instructs the model to re-examine the region named by the question, $\textsf{rea}$ (reason) runs and aggregates further critique passes, and $\textsf{enc}$ (re-encode) recomputes $\bz$ from a higher-resolution crop (Appendix~\ref{app:actions}, Table~\ref{tab:actions}). A \emph{language-side} action is anything the system does once $\bz$ is fixed. Formally, it is a Markov kernel acting on $(\bz,Q,R)$, and we write $\Lc$ for the set of such actions. Thus $\{\textsf{stop},\textsf{att},\textsf{rea}\}\subset\Lc$, while $\textsf{enc}$ lies outside $\Lc$ because it recomputes $\bz$. The set $\Lc$ is closed under composition and finite adaptivity, so it also includes procedures outside $\Ac$, such as self-consistency over any number of critiques (Definition~\ref{def:ls}). We call the unit scored by the verifier an \emph{item}, either a pair or a single image with its question and candidate step. The verifier's \emph{correctness} on an item is the probability it assigns to the true label, $f$ when $y=1$ and $1-f$ when $y=0$, and on a pair it is the members' average $\tfrac12(p_{i,0}+1-p_{i,1})$. We write $\Delta_i(a)$ for the change in the verifier's correctness on item $i$ after action $a$. Each action has a measured cost $c(a)$ with $c(\textsf{stop})=0$, charged at a price $\lambda\ge0$, and its gain is $g_i(a)=\Delta_i(a)-\lambda\,c(a)$. A \emph{policy} $\pi$ chooses one action per item from the information available to it.

\section{Certified interface aliases}\label{sec:cert}

\begin{assumption}[Interface factorization]\label{as:iface}
Fix a pipeline stage $k$.
\begin{enumerate}[label=(\alph*),ref=\thetheorem(\alph*),leftmargin=1.6em,itemsep=0pt,topsep=1pt,parsep=0pt]
\item\label{as:factor} \emph{Factorization.} The conditional law of $\bz$ given $\bs_k(\bx)$ does not depend on $\bx$. No stage after $k$ reads the raw image or any earlier state except through $\bs_k$.
\item\label{as:shared} \emph{Shared non-image inputs.} The members of a pair use identical token sequences for $Q$ and $R$ and the same decoding parameters. Any sampling randomness is drawn independently of $(\bx,\bz)$.
\end{enumerate}
\end{assumption}

\begin{definition}[Certified interface alias]\label{as:hash}
A pair $(\bx_0,\bx_1,Q,R)$ is an \emph{interface alias at stage $k$} if its members produce exactly the same state at stage $k$, \mbox{$\bs_k(\bx_0)=\bs_k(\bx_1)$}. It is a \emph{certified interface alias}, or a \emph{certified pair} for short, once this equality is verified by comparing cryptographic digests of the full serialized state, including array contents, data type, shape and all non-array metadata.
\end{definition}

\begin{theorem}[Interface aliasing certificate]\label{thm:cert}
Under Assumption~\ref{as:iface}, the joint law of $(\bz,Q,R,\omega)$ is the same under both members of an interface alias at stage $k$. Every verifier therefore induces the same output law on both members, and every language-side action preserves that law. Consequently, (i) every decision rule based on that law has $\mathrm{BA}=1/2$ exactly; (ii) $\Delta_i(a)=0$ exactly for every $a\in\Lc$, including $\textsf{att}$, $\textsf{rea}$ with any number of passes, self-consistency over any number of critiques, and deterministic re-supply of the same image; and (iii) any action, language-side or not, that leaves $\bs_k$ shared and preserves Assumption~\ref{as:iface} has $\Delta_i(a)=0$. Therefore, $\Delta_i(a)>0$ requires an action that gives the members different states at stage $k$ or introduces new image-dependent information.
\end{theorem}

Theorem~\ref{thm:cert} is the data-processing inequality in its degenerate case, where a Markov kernel applied to identical inputs produces identical output laws. Its value is that an exact digest comparison establishes the hypothesis without model weights or queries. Assumption~\ref{as:factor} is the theorem's architectural requirement. It excludes a parallel high-resolution branch or raw-image side channel unless the collision occurs before that path branches. Certification also requires equality over every image-dependent field emitted by the processor, so we include all enumerated fields in a single digest (Appendix~\ref{app:complete}, Table~\ref{tab:fields}). The certificate assumes that this enumeration is complete. We verify that every enumerated field affects the digest by perturbing each one in isolation.

Three consequences follow. Proofs are in Appendix~\ref{app:proofs}.

\begin{proposition}[Sharpness]\label{prop:sharp}
Fix a pair, write $\mu_b$ for the law of $\bz$ under member $b$, and let $\mathrm{BA}=\tfrac12(p_0+1-p_1)$ be that pair's value. \emph{(a)} Every verifier reading only the interface, including after any language-side action, has $\mathrm{BA}=1/2$ if and only if $\mu_0=\mu_1$. \emph{(b)} The highest pair-balanced accuracy attainable from the interface is $\tfrac12\bigl(1+\TV(\mu_0,\mu_1)\bigr)$, where the total-variation distance is $\TV(\mu_0,\mu_1)=\sup_A\lvert\mu_0(A)-\mu_1(A)\rvert$. Thus, if $\mu_0\neq\mu_1$, some verifier outperforms chance. If the interface is deterministic and $\bz_0\neq\bz_1$, some verifier attains $\mathrm{BA}=1$.
\end{proposition}

Part (b) is the Neyman--Pearson identity for testing between two laws \citep{lecam1986,tsybakov2009}, so for a deterministic interface only an exact collision certifies chance performance without assumptions about the verifier.

\begin{corollary}[Screen and ceiling]\label{cor:screen}
\begin{enumerate}[label=(\alph*),ref=\thetheorem(\alph*),leftmargin=1.6em,itemsep=0pt,topsep=1pt,parsep=0pt]
\item\label{cor:screen-part} On a certified pair, for every $\lambda>0$ and any costs satisfying $c(\textsf{att}),c(\textsf{rea})>0$, $g_i(\textsf{att})<g_i(\textsf{stop})$ and $g_i(\textsf{rea})<g_i(\textsf{stop})$. Thus, $\textsf{att}$ and $\textsf{rea}$ are strictly dominated by $\textsf{stop}$.
\item\label{cor:ceiling-part} If a fraction $\kappa$ of a population of pairs is certified, every verifier, before or after any language-side action, satisfies $\mathrm{BA}\le1-\kappa/2$. Equality holds if the verifier separates every uncertified pair perfectly.
\end{enumerate}
\end{corollary}

Corollary~\ref{cor:screen-part} provides a direct screen for routing decisions. On a certified pair, choosing $\textsf{att}$ or $\textsf{rea}$ wastes compute, and applying the screen requires only the router's initial decision, without running an intervention or observing its outcome.

\paragraph{Routing value} The routing value measures the advantage of an oracle that chooses the best action for each item over the best fixed action. Write
\begin{equation}\label{eq:V}
V=\Ex_i\big[\max_a g_i(a)\big]-\max_a\Ex_i\big[g_i(a)\big].
\end{equation}
At $\lambda=0$, $g_i(a)=\Delta_i(a)$. Attaining $V$ requires distinguishing items that benefit from different actions. A router sees only a \emph{signal} $S=\sigma(\bz,Q,R)$, some statistic of the interface and the text, and it is \emph{interface-only} when $S$ is a function of $\bz$ alone. A router cannot always attain $V$ because it must use the same action distribution for every item in a \emph{fiber} of its signal, defined as the set of items that map to the same signal value. Proposition~\ref{prop:fiber} gives the resulting ceiling, which is the value-of-information inequality in our notation \citep{blackwell1953,howard1966}.

\begin{proposition}[Fiber ceiling]\label{prop:fiber}
Let $\pi(S,\omega)$ be any policy reading a signal $S$, where $\omega$ is auxiliary randomness independent of the item, and define $V_S=\Ex\big[\max_a\Ex[g_i(a)\mid S_i]\big]-\max_a\Ex[g_i(a)]$. Then \emph{(i)} every such policy satisfies $\Ex[g_i(\pi(S_i,\omega))]-\max_a\Ex[g_i(a)]\le V_S$, with equality attained by $\pi_S(S_i)\in\argmax_a\Ex[g_i(a)\mid S_i]$; \emph{(ii)} $0\le V_S\le V$, and $V_{S'}\le V_S$ whenever $S'$ is a function of $S$; and \emph{(iii)} $V_S=V$ if and only if some policy reading $S$ selects a per-item optimal action almost surely.
\end{proposition}

Thus, $V_S$ is the greatest improvement over the best fixed action available to a router reading $S$, and we write $V_\Phi$ for the interface-only case $S=\bz$. The gap $V-V_S$ is the routing value hidden by variation within fibers. A certified pair witnesses a fiber with two oppositely labeled items for a particular deployed $\Phi$ under every signal, because its members share $\bz$, $Q$ and $R$, so it screens any router.\looseness=-1

\section{\AF{}: constructing and auditing certified aliases}\label{sec:af}

\AF{} analyzes a deployed processor and, when possible, constructs two images that have different correct answers but are mapped by its resize to exactly the same output. It then verifies equality across the complete visual interface and records the evidence needed to certify the pair. Algorithm~\ref{alg:af} gives the complete audit and construction procedure.

A pixel-offset pattern computed in the real-valued null space generally contains fractional entries, which must be rounded to form a valid image. None of the $72$ rounded candidates we tested as a control certified, with roughly $1\%$ of their output values remaining unequal (Appendix~\ref{app:e0}). We therefore construct the offset patterns using the resampler's exact integer arithmetic. The resize implementations selected by the target processors represent their filter weights as fixed-point integers rather than floating-point numbers and differ only in final rounding. Their pre-rounding computations can therefore be represented exactly by an integer matrix $K$.

The resize is \emph{separable}, meaning that it applies one-dimensional filtering independently in the horizontal and vertical directions. \AF{} therefore analyzes the horizontal operation. Each column of $K$ describes how the corresponding input-image column contributes to the resized output. Because each output column depends only on nearby input columns, the search can be restricted to $W$ adjacent input columns, which we call a \emph{window}. Let $K_W$ denote the corresponding restriction of $K$. The \emph{kernel} $\ker K_W$ is the set of offset patterns that $K_W$ maps to zero. \AF{} searches this kernel for a nonzero integer pattern $v\in\ker K_W\cap\mathbb{Z}^{W}$. Each entry of $v$ specifies how much to change the pixels in one column of the window. Multiplying $v$ by an integer increases its contrast without moving it out of the kernel. Repeating the resulting offsets across selected image rows creates a visible, label-changing pattern. Because the horizontal resize maps every modified row to zero, the subsequent vertical resize also receives zero difference, leaving the final resized output unchanged. We call this pattern the \emph{carrier}. It can be added to any image whose affected pixels have enough range to avoid clipping, including natural images and synthetic renders.

Our construction is not available at every resize ratio. It requires a \emph{realizable} vector, a nonzero integer vector in $\ker K_W$ whose largest absolute entry is at most $255$, the largest possible difference between two 8-bit pixel values. Whether such a vector exists depends on both the resize ratio and the resampling filter. These vectors form the integer lattice $L=\ker K_W\cap\mathbb{Z}^{W}$, a discrete set closed under integer combinations. Because $L$ is determined entirely by the deployed resize configuration, constructibility can be evaluated without model weights or queries. Away from the image border, a resampler at ratio $p/q$ in lowest terms is periodic, so it is described by a $q\times p$ matrix $H(\zeta)$ of integer polynomials in an indeterminate $\zeta$, its \emph{polyphase matrix} \citep{vaidyanathan1993}, and an integer offset vector $v$ on interior columns by a polynomial vector $\hat v(\zeta)\in\mathbb{Z}[\zeta]^{p}$ (Appendix~\ref{app:proofs}).

\begin{theorem}[Existence and size of integer kernel vectors]\label{thm:poly}
Let $K$ be the fixed-point resize operator at ratio $p/q>1$, and let $v$ be an integer offset vector on interior columns (Remark~\ref{rem:border}).
\begin{enumerate}[label=(\alph*),ref=\thetheorem(\alph*),leftmargin=1.6em,itemsep=0pt,topsep=1pt,parsep=0pt]
\item\label{poly:char} \emph{Existence.} $Kv=0$ if and only if $H(\zeta)\hat v(\zeta)=0$. The solutions form a $\mathbb{Z}[\zeta]$-module of rank at least $p-q\ge1$, so every wide enough image admits finitely supported kernel vectors $v\neq0$.
\item\label{poly:principal} \emph{Principal case.} If $p=q+1$ and $H$ has full row rank, every integer kernel vector is $\hat v=\xi\,w$ for some $\xi\in\mathbb{Z}[\zeta]$, where $w$ is the vector of signed maximal minors of $H$ divided by their greatest common divisor. Every nonzero interior kernel vector of any support width then satisfies $\lVert v\rVert_\infty\ge\beta(K):=\max_{s:\,w_s\neq0}\max\{\lvert\mathrm{lc}(w_s)\rvert,\lvert\mathrm{tc}(w_s)\rvert\}$, where $\mathrm{lc}$ and $\mathrm{tc}$ are the highest- and lowest-order nonzero coefficients, and $w$ itself is a kernel vector of sup-norm $\lVert w\rVert_\infty$.
\end{enumerate}
\end{theorem}

Non-constructibility is therefore always a question of size. The principal case covers $2\times$, $1.5\times$ and $4/3\times$, and the rank condition holds for every operator we reconstruct (Table~\ref{tab:beta}). For bicubic resampling at $2\times$, $w$ encodes the carrier $(1,-3,3,-1)$ and $\beta(K)=3=\lVert w\rVert_\infty$, so our carrier is the smallest possible at any interior support width. For Lanczos at $2\times$, $\beta(K)=63{,}150$, and at $1.5\times$ it exceeds $5\times10^{9}$ for bicubic, bilinear and Lanczos, so no realizable interior carrier exists at any width. Outside the principal case, such as $3\times$ and $2.5\times$, we bound window-supported vectors.

We reduce a basis of $L$ to search for a realizable vector. If none is found, the bound in Proposition~\ref{prop:construct} may certify that no such vector exists within the specified window, because every nonzero $v\in L$ satisfies $\lVert v\rVert_\infty\ge\min_i\lVert b_i^*\rVert/\sqrt{W}$ for the Gram--Schmidt vectors $b_i^*$ of any basis of $L$ (Appendix~\ref{app:proofs}). Otherwise, the configuration remains unresolved. The bound covers the rank-one differences $\bd=uv^{\!\top}$ that \AF{} produces, where the binary vector $u$ selects the affected rows. Appendix~\ref{app:proofs} extends it to differences of any rank, which rules out both declared configurations at $1.5\times$.\looseness=-1

Applied to the $20$ declared configurations using bicubic resampling, the criterion returns a definitive verdict in every case (Appendix~\ref{app:construct}). Seventeen are constructible and three are ruled out. All constructible configurations use integer resize ratios, while the three exclusions use non-integer ratios (Table~\ref{tab:screen}). These verdicts depend only on $K$ and are obtained before an image is rendered or a model is queried.

We also reconstructed $K$ for $28$ resampling-filter and ratio combinations. The reconstruction matched the library bit for bit in $27$ cases, and the criterion returned a definitive verdict in $26$. No tested filter eliminates the construction at every ratio (Table~\ref{tab:kernels}, Appendix~\ref{app:construct}).

\begin{table}[t]
\centering
\footnotesize
\renewcommand{\arraystretch}{0.98}
\caption{\textbf{Instrument validation.} Targets are Qwen2.5-VL-7B \citep{qwen25vl}, LLaVA-1.5-7B \citep{llava15} and VisualPRM-8B \citep{visualprm}. Predictions matched the observed outcomes in all $192$ attempts, and all $18$ measured decision-logit gaps between members were exactly zero.}
\label{tab:e0}
\vspace{3pt}
\setlength{\tabcolsep}{4pt}
\begin{tabular*}{\textwidth}{@{\hspace{\tinset}\extracolsep{\fill}}l L{0.25\textwidth} r@{\extracolsep{0pt}\hspace{0.35em}}l@{\extracolsep{\fill}} L{0.23\textwidth}@{\hspace{\tinset}}}
\thead
\thd{Target or control} & \thd{Resize configuration} & \multicolumn{2}{c}{\thd{Certified pairs}} & \thd{Member logit gap} \\
\theadend
Qwen2.5-VL-7B & $2\times$ downscale & \tkey{cCert}{$15$} & \tkey{cCert}{of $15$} & \tkey{cCert}{zero} on all $8$ scored \\
\thairx{1pt}
LLaVA-1.5-7B & fixed $2\times$ resize & \tkey{cCert}{$10$} & \tkey{cCert}{of $10$} & \tkey{cCert}{zero} on all $10$ \\
\thairx{1pt}
VisualPRM-8B & $3{\times}3$ tiles and thumbnail & \tkey{cCert}{$20$} & \tkey{cCert}{of $20$} & \tmut{not available} \\
\tgroupsep
All negative controls & four control designs & \tkey{cExcl}{$0$} & \tkey{cExcl}{of $147$} & \tkey{cExcl}{nonzero} when available \\
\bottomrule
\end{tabular*}
\setlength{\tabcolsep}{6pt}
\end{table}

\paragraph{Instrument validation} Table~\ref{tab:e0} reports the certification experiment. The constructibility criterion was evaluated before each experiment, so the table compares prospective predictions with observed outcomes. For the second architecture, the criterion predicted constructibility using only the deployed processor, and the construction transferred without modification. The dynamic-tiling target provides a more demanding test because it processes nine $2\times$ tiles and a $6\times$ thumbnail, requiring the pair to collide under both resize operators. All $20$ constructed pairs certified across the nine tiles, the thumbnail, and every recorded interface field, while none of the $20$ same-support, same-amplitude patterns outside the kernel did. An identity-resize control also did not certify, confirming that the construction requires both exact kernel membership and a lossy resize. Four additional pipelines certified at the processor level (Table~\ref{tab:breadth}), and Appendix~\ref{app:e0} gives the full controls.

The construction extends beyond synthetic base images. On $200$ benchmark images resized to the eligible $2\times$ regime, selecting carrier rows with sufficient intensity headroom produced $48$ certified pairs at amplitude $20$, of which $8$ keep the marks clearly legible (Figure~\ref{fig:pair}, Appendix~\ref{app:construct}).

\begin{figure}[t] \centering
\includegraphics[width=\textwidth,trim=0 0 0 4pt,clip]{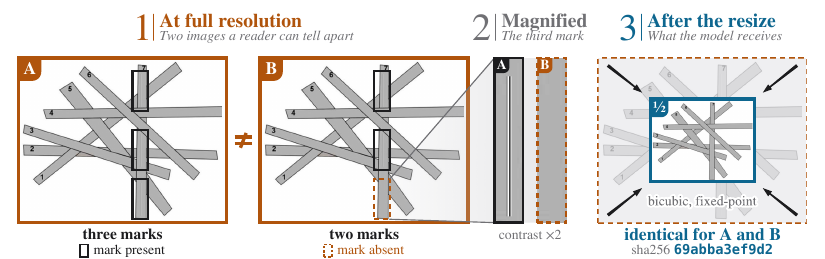} \caption{\textbf{A certified pair on a MathVision image.} (1)~Member A contains three marks and member B contains two, giving different correct counts. (2)~The third mark is magnified at twice its original contrast. (3)~After $2\times$ resizing, the tensors are bit-identical.}
\label{fig:pair}
\end{figure}


\section{Interface accessibility and routing}\label{sec:ias}

To test whether an interface-only router can recognize when relevant visual evidence is unavailable, we need a score computed from the interface alone. We define this score relative to the binary fact associated with each item. On the natural stream of benchmark items, the fact is whether the candidate reasoning step is correct. The question $Q_i$ and step $R_i$ identify the fact being evaluated but are not decoder features. An auxiliary decoder $d_\theta$, implemented as logistic regression on a projection of $\bz$ (Appendix~\ref{app:protocol}), estimates the probability $\hat p_\theta(1\mid\bz_i)$ that the fact is positive. The \emph{interface accessibility score} is $\IAS_i=\mathrm{cal}\!\left(\hat p_\theta(1\mid\bz_i)\right)$, where $\mathrm{cal}$ is an out-of-fold isotonic calibration map. We cross-fit the decoder with folds grouped by source image and, on the natural stream, by \emph{record}, the source problem containing each item's step. The certified set provides an absent-side calibration anchor at $\IAS=1/2$.

For each action, a linear model fitted on a fixed random half of the records, the development half, estimates its gain $\hat g_i(a)$ from $\IAS_i$. The router \ROUTE{} selects the action with the largest estimated gain. At test time, this routing decision depends only on $\bz$. The question and candidate step are passed to the selected action after the decision is made.

\paragraph{Evaluation metrics} A policy's regret is its expected shortfall relative to the per-item oracle $\pi^\star$,
\begin{equation}\label{eq:gain}
\mathrm{regret}(\pi)=\Ex_i\big[g_i(\pi^\star(i))-g_i(\pi(i))\big],\qquad
\pi^\star(i)=\argmax_{a\in\Ac} g_i(a).
\end{equation}
Order accuracy is the fraction of items on which $\pi$ selects the same action as $\pi^\star$. The dominated-action rate is the fraction of certified pairs sent to $\textsf{att}$ or $\textsf{rea}$, which gain nothing there and are strictly dominated for every $\lambda>0$ (Corollary~\ref{cor:screen-part}). All headline results use raw gains with $\lambda=0$. We also evaluate the sweep $\lambda\in\{0,0.005,0.01,0.02,0.05,0.1\}$ using the measured action costs in Appendix~\ref{app:actions}.

\paragraph{Experimental design} Because a deployed system does not know which items it answered incorrectly, we evaluate every item rather than only observed errors. To isolate the effect of re-examining the image, gains are measured against a \emph{path-matched} baseline that follows the same critique procedure but omits the instruction to re-examine the image (Appendix~\ref{app:extra}). To avoid selection bias, each item is measured twice, with per-item maxima and regrets selected on one replicate and evaluated on the other (Appendix~\ref{app:bookkeeping}, Table~\ref{tab:ladder}). The primary endpoint tests whether a score calibrated on certified pairs can guide routing on natural items. Its pre-registered routing-headroom gate was powered to detect $V\ge0.03$ using $160$ templates of eight items (Appendix~\ref{app:power}, Table~\ref{tab:power}). All remaining tests are secondary and Holm-corrected within each family, and intervals are clustered by record.

\section{Routing evaluation}\label{sec:results}

The instrument-validation experiments in Section~\ref{sec:af} test four pre-registered claims, all of which pass. The validation experiments show that the criterion correctly predicts which resize configurations admit exact collisions, rounded approximate constructions fail certification, certified pair members produce identical decision logits, and the digest detects changes to every recorded interface field. The primary experiment tests a fifth claim that a label-free interface score calibrated to chance when relevant evidence is provably absent can guide routing on natural items, but the score does not improve over always stopping. Two pre-registered gates make this negative result interpretable. They require measurable headroom for per-item routing and a probe that recovers known-present information without fitting shuffled labels. Both pass, ruling out insufficient routing opportunity and failure on either probe control as simple explanations.

\begin{table}[!ht]
\centering
\footnotesize
\renewcommand{\arraystretch}{0.92}
\caption{\textbf{Certified-pair screen and natural-stream routing.} The screen reports dominated-action rates on $80$ certified Qwen2.5-VL-7B pairs (Corollary~\ref{cor:screen-part}). Natural results use $1280$ VisualProcessBench items from $1018$ records with a frozen Qwen2.5-VL-7B verifier. Gain is relative to stopping at $\lambda=0$, with $95\%$ intervals clustered by record. $^\dagger$Best of six, Holm-adjusted $p=0.11$.}
\label{tab:main}
\vspace{3pt}
\setlength{\tabcolsep}{3pt}
\newcommand{\ci}[1]{{\scriptsize\tmut{$[#1]$}}}
\newcommand{\ours}{\,\tmut{(ours)}}
\begin{tabular*}{\textwidth}{@{\hspace{\tinset}\extracolsep{\fill}}l
    S[table-format=1.3]
    @{\hspace{2.2em}}
    S[table-format=+1.3,retain-explicit-plus=true]
    @{\extracolsep{0pt}\hspace{0.5em}}l@{\extracolsep{\fill}}
    S[table-format=1.3]
    S[table-format=1.3]@{\hspace{\tinset}}}
\theadn{3}{3.6pt}
& \multicolumn{1}{c}{\thd{Certified pairs}} & \multicolumn{4}{c@{\hspace{\tinset}}}{\thd{Natural stream}} \\
\tpartial{\cmidrule(lr){2-2}\cmidrule(l){3-6}}
& \multicolumn{1}{c}{\thd{Dominated}} & \multicolumn{2}{c}{\thd{Gain over \textsf{stop} $\uparrow$}} & & \multicolumn{1}{c@{\hspace{\tinset}}}{\thd{Order}} \\
\thd{Routing method}
& \multicolumn{1}{c}{\thd{action rate $\downarrow$}}
& \multicolumn{2}{c}{\thdsub{with 95\% interval}}
& \multicolumn{1}{c}{\thd{Regret $\downarrow$}}
& \multicolumn{1}{c@{\hspace{\tinset}}}{\thd{accuracy $\uparrow$}} \\
\theadend
\tgroup{6}{Fixed policies}{}
Always \textsf{stop}      & \color{cCert}0.000 &  0.000 & {\scriptsize\tmut{reference}}  & 0.105 & 0.271 \\
\thairx{0.4pt}
Always \textsf{re-attend} & \color{cExcl}1.000 & +0.002 & \ci{-0.008,\,+0.010} & 0.103 & 0.281 \\
\thairx{0.4pt}
Always \textsf{reason}    & \color{cExcl}1.000 &  0.000 & \ci{-0.007,\,+0.008} & 0.104 & 0.188 \\
\thairx{0.4pt}
Always \textsf{re-encode} & \color{cCert}0.000 & -0.021 & \ci{-0.035,\,-0.008} & 0.127 & 0.260 \\
\tgroupsep
\tgroup{6}{Protocol-specified label-free policies}{}
Verifier confidence      & \color{cExcl}1.000 & -0.001 & \ci{-0.009,\,+0.007} & 0.106 & 0.202 \\
\thairx{0.4pt}
\ROUTE{}\ours           & \color{cExcl}1.000 & +0.001 & \ci{-0.007,\,+0.009} & 0.104 & 0.188 \\
\tgroupsep
\tgroup{6}{Median-threshold diagnostic}{}
Verifier confidence      & \color{cExcl}0.838 & +0.015 & \ci{+0.008,\,+0.023} & 0.089 & 0.344 \\
\thairx{0.4pt}
Raw interface distance   & \color{cExcl}0.838 & +0.021 & \ci{+0.014,\,+0.028} & 0.084 & 0.369 \\
\thairx{0.4pt}
Random score             & \color{cExcl}0.414 & +0.018 & \ci{+0.009,\,+0.026} & 0.087 & 0.368 \\
\thairx{0.4pt}
\ROUTE{}\ours           & \color{cExcl}0.838 & +0.016 & \ci{+0.009,\,+0.024} & 0.089 & 0.366 \\
\tgroupsep
\tgroup{6}{Oracle-supervised references}{}
Learned router on measured $\Delta$, from $\bz$ & \color{cExcl}0.063 & +0.003 & \ci{-0.006,\,+0.012} & 0.102 & 0.259 \\
\thairx{0.4pt}
Learned router, from $(\bz,Q,R)$\,$^\dagger$ & \color{cCert}0.000 & +0.010 & \ci{+0.002,\,+0.018} & 0.095 & 0.283 \\
\thairx{0.4pt}
Per-item oracle (split-sample)       & \color{cCert}0.000 & +0.058 & \ci{+0.050,\,+0.067} & 0.047 & 0.659 \\
\bottomrule
\end{tabular*}
\setlength{\tabcolsep}{6pt}
\end{table}

\paragraph{Routing headroom} The best action varies across items even though no single action improves on stopping on average. The split-sample routing value is $V=0.056$ over the best fixed action, with a record-clustered $95\%$ interval of $[0.047,0.062]$, and the per-item oracle gains $+0.058$ over stopping (Table~\ref{tab:main}). Against the path-matched baseline, re-attention gains $+0.002$ and further reasoning gains $0$, both with intervals covering zero, while re-encoding loses $0.021$ with an interval of $[-0.035,-0.008]$. Thus, positive routing value reflects item-level variation rather than an action that helps uniformly. By contrast, a split-sample ten-crop oracle gains $+0.035$ over stopping, with a $95\%$ interval of $[+0.022,+0.047]$, while the blind crop loses $0.021$ \citep{vicrop}. This shows additional headroom from selecting the right interface-changing view, but only with oracle supervision.\looseness=-1

\paragraph{Probe controls and the primary endpoint} The probe, our decoder $d_\theta$, behaves as expected at all three anchors. Its AUC is $0.9999$ when the fact survives the interface, $0.52$ under random labels and exactly $1/2$ on certified pairs, where chance is guaranteed. Yet \ROUTE{} does not improve over stopping under the protocol-specified rule, with a gain of $+0.001$ and an interval of $[-0.007,+0.009]$, and neither does verifier confidence. The median-threshold diagnostic produces gains of $+0.015$ to $+0.021$ across \IAS{}, confidence, a random score and raw interface distance. Because the random score performs similarly and all gains disappear when the non-stop action is fixed, the improvement comes from replicate-informed action selection rather than the scores. The same pattern holds across all five source benchmarks. Each has positive routing value, but neither \ROUTE{} nor the learned router significantly beats its best fixed action (Table~\ref{tab:sources}).

Proposition~\ref{prop:fiber} explains how positive routing value can coexist with this result. An interface-only router is limited to $V_\Phi$, which may be smaller than the full per-item value $V$. Six oracle-supervised routers trained directly on natural-item gains using the vision tower's outputs achieve only $-0.003$ to $+0.003$ over stopping, with every interval covering zero (Table~\ref{tab:routers}). Thus, poor transfer from certified pairs cannot by itself explain why the routers fail to improve over stopping. These gains only lower-bound $V_\Phi$ and do not establish that $V_\Phi=0$. Six further routers read the verifier's hidden state at the decision position, a signal of $(\bz,Q,R)$. The best gains $+0.010$, about one sixth of $V$, with an interval of $[+0.002,+0.018]$ that does not survive Holm correction across the six. The natural evidence remains behavioral. The verdict is unchanged across three critiques on $45.2\%$ of $2560$ measurements and $49.6\%$ of $954$ errors. When it changes, repairs and breakages offset one another, contributing $+0.030$ and $-0.016$ for re-attention (Appendix~\ref{app:extra}, Table~\ref{tab:a1}).

\paragraph{The screen in practice} On a certified pair, Corollary~\ref{cor:screen-part} proves that $\textsf{att}$ and $\textsf{rea}$ are dominated by stopping for every $\lambda>0$. A router that recognizes this limit should avoid both actions, yet none of the label-free scores does so. Under the median rule, every score except the random control acts on all $80$ certified pairs but only half of the natural stream. The random control acts on half of each. The protocol-specified rule acts on all items and sends every certified pair to reasoning. Routers trained on the vision tower's outputs behave less uniformly (Table~\ref{tab:routers}). One stops on all $80$ pairs. Another acts on $70$ but selects re-encoding, which the corollary does not rule out, on $65$ of them. The remaining four select a dominated action on between $6$ and $80$ pairs. Five of the six hidden-state routers stop on all $80$. The screen therefore distinguishes policies that avoid provably wasted computation from those that do not.

\paragraph{Limitations} All certified witnesses use one resize-based mechanism in one modality, so broader applications of Proposition~\ref{prop:fiber} remain untested. Our exclusions cover only exact-kernel collisions that can be added to any base image (Remark~\ref{rem:quant}). The construction requires a fixed-point resampler whose filter and resize ratio yield an integer kernel vector that fits within the pixel range. Among the bicubic configurations we resolve, only integer resize ratios are constructible, and the construction is unavailable when a deployment forwards images at native resolution \citep{navit,qwen2vl}. An adversary can choose an eligible input resolution, but natural images rarely arrive at one. We therefore resized the benchmark images into the eligible regime, where only a small minority supported a legible mark. We do not measure how often natural errors are interface-limited, and identifying such errors would require mining natural failures. Every certificate also assumes that all image-dependent interface fields have been enumerated correctly (Appendix~\ref{app:complete}). Finally, because the primary endpoint is null, we do not ablate the components of \IAS{}.

\section{Conclusion}

Perceptual failures are usually inferred from behavior. When pre-rounding resize is an exact integer map, a certified pair makes one class provable: members with different correct answers have identical interface states, so no verifier or language-side repair can beat chance. From the resize configuration, \AF{} either constructs such a pair or certifies its absence within the specified family. Certified pairs also expose routers that waste compute on re-attention or further reasoning. Natural items exhibit routing headroom, but no tested interface-only router improves over stopping. The fiber ceiling explains how this can occur: a router cannot distinguish inputs within a fiber (Proposition~\ref{prop:fiber}). Extending certification beyond fixed-point resampling is the next step.\looseness=-1

\bibliography{references}

\appendix
\renewcommand{\topfraction}{0.8}
\section{Proofs}\label{app:proofs}

\begin{definition}[Language-side class]\label{def:ls}
Fix a question and candidate step $(Q,R)$. A \emph{language-side intervention} is a measurable function $\psi(\bz,Q,R,\omega)$, where the randomness $\omega$ is drawn independently of $(\bx,\bz)$. A finite adaptive procedure also belongs to this class because its final output remains a measurable function of $(\bz,Q,R,\omega)$. Re-supplying the identical image to a deterministic processor is language-side, and $\textsf{enc}$ is not because it recomputes $\bz$.
\end{definition}

\begin{proof}[Proof of Proposition~\ref{prop:sharp}] Let $\nu=\mu_0+\mu_1$, write $q_0$ and $q_1$ for the corresponding densities, and define $A^\star=\{q_0>q_1\}$. A deterministic verifier whose positive decision region is $A$ has $\mathrm{BA}={\tfrac12\bigl(1+\mu_0(A)-\mu_1(A)\bigr)}$, and for every measurable $A$,
\[
\mu_0(A)-\mu_1(A)
=\int_A(q_0-q_1)\,d\nu
\le\int_{A^\star}(q_0-q_1)\,d\nu
=\TV(\mu_0,\mu_1),
\]
with equality when $A=A^\star$. Randomized rules are mixtures of deterministic rules and cannot exceed this maximum, which proves (b). If the interface is deterministic and $\bz_0\neq\bz_1$, then $\mu_0$ and $\mu_1$ are distinct point masses and $\TV(\mu_0,\mu_1)=1$. For (a), sufficiency holds because a verifier, before or after a language-side action, is a Markov kernel on $(\bz,Q,R)$, so $\mu_0=\mu_1$ gives $p_0=p_1$, and necessity follows from (b) because $\TV(\mu_0,\mu_1)>0$ whenever $\mu_0\neq\mu_1$.
\end{proof}

\begin{proof}[Proof of Corollary~\ref{cor:ceiling-part}] Partition the population into the certified fraction $\kappa$ and the remainder. Every verifier attains exactly $\tfrac12$ on the first by Theorem~\ref{thm:cert}(i) and at most $1$ on the second, so the aggregate is at most $\kappa\cdot\tfrac12+(1-\kappa)\cdot1=1-\kappa/2$, with equality when the remainder is answered perfectly.
\end{proof}

\begin{proof}[Proof of Corollary~\ref{cor:screen-part}] By Theorem~\ref{thm:cert}(ii), $\Delta_i(a)=0$ for every $a\in\{\textsf{stop},\textsf{att},\textsf{rea}\}$ on a certified pair, so $g_i(a)=-\lambda c(a)$ for these actions. Because $c(\textsf{stop})=0$, $c(\textsf{att}),c(\textsf{rea})>0$ and $\lambda>0$, we have $g_i(\textsf{att}),g_i(\textsf{rea})<0=g_i(\textsf{stop})$. Within $\Ac$, stopping is optimal exactly when $\Delta_i(\textsf{enc})\le\lambda c(\textsf{enc})$, which is why the screen treats $\textsf{enc}$ separately.
\end{proof}

\begin{proof}[Proof of Theorem~\ref{thm:cert}]
By Definition~\ref{as:hash}, both members produce the same state $\bs_k$. Assumption~\ref{as:factor} then gives the same conditional law of $\bz$ under both members. Together with Assumption~\ref{as:shared}, this implies that the joint law of $(\bz,Q,R,\omega)$ is identical. By Definition~\ref{def:ls}, every language-side intervention is a measurable function of this tuple, so its push-forward law is also identical.

For any binary rule applied to this common output law, let $p$ be the probability of predicting label $1$. Because the two members have labels $1$ and $0$, $\mathrm{BA}=\tfrac12\bigl(p+(1-p)\bigr)=\tfrac12$, which proves (i). The same calculation holds before and after any language-side action, so $\Delta_i(a)=0$, proving (ii). For (iii), an action that produces a shared post-action state at stage $k$ and preserves Assumption~\ref{as:iface} yields an interface alias at stage $k$, so the same argument gives $\Delta_i(a)=0$.
\end{proof}

\begin{remark}[Which signal the ceiling refers to]
Proposition~\ref{prop:fiber} is stated for a signal $S=\sigma(\bz,Q,R)$ because the bound depends on what the router reads. Let $\bz$ be constant and $Q$ be uniform on $\{Q',Q''\}$, and suppose action $a_1$ gains $1$ under $Q'$ and $0$ under $Q''$, with the gains reversed for $a_2$. Then $V_\Phi=0$, but a policy reading $Q$ gains $1/2$ over the best fixed action. Conversely, in our natural stream the tuple $(\bz,Q,R)$ identifies the item, so $V_{(\bz,Q,R)}=V$ and the ceiling for that signal is vacuous without restricting the router class. \IAS{}, interface distance and the vision-tower routers read $\bz$ alone, while verifier confidence and the hidden-state routers read $(\bz,Q,R)$.
\end{remark}

\begin{proposition}[Non-constructibility bound]\label{prop:construct}
Fix a deployed resize with fixed-point integer coefficient matrix $K$ and a window of $W$ input columns. Let $K_W$ be the restriction of $K$ to that window, and let $L=\ker K_W\cap\mathbb{Z}^{W}$. The construction forms the image difference $\bd=uv^{\!\top}$ inside the window and sets it to zero elsewhere, where $u$ is a binary vector selecting the affected image rows and $v\in L$ supplies the column offsets. We call $v$ \emph{realizable} when $v\neq0$ and $\lVert v\rVert_\infty\le255$. For any basis of $L$ with Gram--Schmidt vectors $b_1^*,\dots,b_n^*$, every nonzero $v\in L$ satisfies $\lVert v\rVert_2\ge\min_i\lVert b_i^*\rVert$ and therefore $\lVert v\rVert_\infty\ge\min_i\lVert b_i^*\rVert/\sqrt{W}$. Consequently, if $\min_i\lVert b_i^*\rVert>255\sqrt{W}$, no realizable perturbation in this rank-one family can be supported on that window.
\end{proposition}

The bound follows because the shortest nonzero vector of a lattice is at least as long as the shortest Gram--Schmidt vector of any of its bases \citep{lll}. It is basis-dependent, window-local and one-sided. Exceeding the threshold certifies impossibility, while falling below it establishes only that a search may succeed.

\begin{proof}[Arbitrary-rank extension of Proposition~\ref{prop:construct}]
Assume the same fixed-point operator $K_W$ acts along both image axes, and let $\bd$ be an integer difference whose support the window covers along both axes. The deployed resampler applies $K_W$ one axis at a time, so a difference that collides on every unclipped base is annihilated by the first pass and therefore satisfies $K_W\bd K_W^{\!\top}=0$. We bound this larger exact-kernel family. The rank-one family $uv^{\!\top}$ does not exhaust the solutions. For example, the integer kernel of $G=\left[\begin{smallmatrix}160&96&0\\0&96&160\end{smallmatrix}\right]$ is spanned by $(3,-5,3)$, yet $\bd=\left[\begin{smallmatrix}0&-1&0\\1&0&1\\0&-1&0\end{smallmatrix}\right]$ has rank $2$ and sup-norm $1$ and satisfies $G\bd G^{\!\top}=0$. Define
\[
\tau_\infty(L)=\min_{v\in L\setminus\{0\}}\lVert v\rVert_\infty
\quad\text{and}\quad
M=\max_a\lVert\operatorname{row}_a(K_W)\rVert_1,
\]
and let $Y=K_W\bd$. Since $YK_W^{\!\top}=0$, the transpose of every row of $Y$ lies in $L$. If $Y=0$, every column of $\bd$ lies in $L$, so $\lVert\bd\rVert_\infty\ge\tau_\infty(L)$. If $Y\neq0$, a nonzero row $Y_a=\operatorname{row}_a(K_W)\bd$ satisfies $\tau_\infty(L)\le\lVert Y_a\rVert_\infty\le\lVert\operatorname{row}_a(K_W)\rVert_1\lVert\bd\rVert_\infty\le M\lVert\bd\rVert_\infty$. Because $M\ge1$, both cases give $\lVert\bd\rVert_\infty\ge\tau_\infty(L)/M$.

Substituting the lower bound of Proposition~\ref{prop:construct} for $\tau_\infty(L)$, the resulting bounds at $1.5\times$ are $3.3\times10^5$ for bicubic, $2.9\times10^5$ for bilinear, and $3.1\times10^5$ for Lanczos, all above $255$, so these configurations admit no nonzero window-supported integer collision of any rank in the exact kernel. In the principal case of Theorem~\ref{thm:poly}, every row of $Y$ and, when $Y=0$, every column of an interior $\bd$ is an interior kernel vector, so $\tau_\infty(L)$ can be replaced by $\beta(K)$. At $1.5\times$, this gives $5.7\times10^{3}$, $1.2\times10^{6}$ and $964$ for any interior support (Table~\ref{tab:beta}). The exclusion at $2.5\times$ and the Lanczos exclusions in Table~\ref{tab:kernels} apply only to the rank-one family.
\end{proof}

\begin{remark}[Exact kernel and output quantization]\label{rem:quant}
A collision need not arise from exact annihilation. The deployed resize rounds and clips its integer accumulator, so distinct pre-quantization values can produce the same 8-bit output, depending on the base image. A perturbation annihilated by the first resize pass instead collides on every base for which it is realizable without input clipping. Our non-constructibility results rule out only exact-kernel perturbations within the stated construction family.
\end{remark}

\begin{proof}[Proof of Proposition~\ref{prop:fiber}]
For (i), let $p_\pi(a\mid S)$ be the probability that $\pi(S,\omega)$ selects action $a$. Because $\omega$ is independent of the item, $\Ex[g(\pi)\mid S]=\sum_a p_\pi(a\mid S)\Ex[g(a)\mid S]\le\max_a\Ex[g(a)\mid S]$. Taking expectations and subtracting $\max_a\Ex[g(a)]$ gives the bound. Because $\Ac$ is finite, a fixed tie-breaking rule makes $\pi_S(S)\in\argmax_a\Ex[g(a)\mid S]$ measurable, and this policy attains equality.

For (ii),
\[
\max_a\Ex[g(a)]
\le
\Ex\!\left[\max_a\Ex[g(a)\mid S]\right]
\le
\Ex\!\left[\max_a g(a)\right].
\]
The first inequality holds because constant policies read any signal, and the second follows from the conditional expectation of the pointwise maximum. If $S'=h(S)$, the tower property and the same inequality applied conditionally on $S'$ give $V_{S'}\le V_S$, the statement that garbling a signal cannot increase its value \citep{blackwell1953}.

For (iii), $V-V_S=\Ex\!\left[\max_a g(a)-g(\pi_S(S))\right]$ has a nonnegative integrand, so $V_S=V$ exactly when $\pi_S(S)$ selects a per-item optimal action almost surely.
\end{proof}

\begin{proof}[Proof of Theorem~\ref{thm:poly}]
\emph{Indexing.} Number the output columns so that column $j=qn+\varphi$, with $n\in\mathbb{Z}$ and phase $\varphi\in\{0,\dots,q-1\}$, reads input columns $pn+o_\varphi+t$ for $t=0,\dots,T_\varphi-1$ with weights $k_\varphi[t]$. Away from the border, the offsets $o_\varphi$ and the weights depend only on $\varphi$, which is the periodicity used in Section~\ref{sec:af}. Write each input index as $pn+o_\varphi+t=p(n+e)+s$ with residue $s\in\{0,\dots,p-1\}$, and set $h_{\varphi s}[e]=k_\varphi[pe+s-o_\varphi]$, taken as zero outside the weight range. With $v_s[n]=v[pn+s]$, the equation $(Kv)_{qn+\varphi}=0$ reads
\[
\sum_{s=0}^{p-1}\sum_{e}h_{\varphi s}[e]\,v_s[n+e]=0\qquad\text{for all }n\in\mathbb{Z}\text{ and }\varphi .
\]
Define the Laurent polynomials $\tilde H_{\varphi s}(\zeta)=\sum_e h_{\varphi s}[e]\zeta^{-e}$ and $\hat v_s(\zeta)=\sum_n v_s[n]\zeta^{n}$. The left-hand side is the coefficient of $\zeta^{n}$ in $\sum_s\tilde H_{\varphi s}(\zeta)\hat v_s(\zeta)$, so $Kv=0$ on the interior if and only if $\tilde H(\zeta)\hat v(\zeta)=0$. Multiplying row $\varphi$ by a suitable power of $\zeta$ turns $\tilde H$ into the polynomial matrix $H\in\mathbb{Z}[\zeta]^{q\times p}$ of Section~\ref{sec:af} without changing the solution set, and a solution $\hat v$ has finite support if and only if $v$ does.

\emph{Existence.} Over the field $\mathbb{Q}(\zeta)$, $H$ has rank at most $q$, so its null space has dimension at least $p-q\ge1$. Any nonzero rational null vector, multiplied by the least common multiple of its denominators, is a nonzero vector in $\mathbb{Z}[\zeta]^{p}$ annihilated by $H$. The set of all such vectors is a $\mathbb{Z}[\zeta]$-submodule of $\mathbb{Z}[\zeta]^{p}$ that spans the null space, so its rank equals that dimension. Shifting a solution into the interior of the image gives a finitely supported integer kernel vector of $K$, provided the interior is at least as wide as its support.

\emph{Principal case.} Let $p=q+1$ and suppose $H$ has full row rank, so the null space over $\mathbb{Q}(\zeta)$ is one-dimensional. Let $r_s=(-1)^{s}\det H^{(s)}$, where $H^{(s)}$ deletes column $s$. For each row $\varphi$, $\sum_s H_{\varphi s}r_s$ is the Laplace expansion of the determinant of the $(q+1)\times(q+1)$ matrix obtained by prepending row $\varphi$ of $H$ to $H$. That matrix has a repeated row, so $Hr=0$, and $r\neq0$ because some maximal minor is nonzero. Let $\gamma=\gcd(r_0,\dots,r_{p-1})$ in $\mathbb{Z}[\zeta]$, including integer content, and $w=r/\gamma$, so that $\gcd(w_0,\dots,w_{p-1})=1$. Now let $\hat v\in\mathbb{Z}[\zeta]^{p}$ be any integer solution. Since the null space is one-dimensional, $\hat v=(\xi_1/\xi_2)\,w$ with $\xi_1,\xi_2\in\mathbb{Z}[\zeta]$ coprime and $\xi_2\neq0$. Then $\xi_2$ divides $\xi_1 w_s$ for every $s$, so by Gauss's lemma $\xi_2$ divides every $w_s$ and hence their greatest common divisor $1$. Thus $\xi_2=\pm1$ and $\hat v=\xi w$ with $\xi=\pm\xi_1$. For the bound, fix $s$ with $w_s\neq0$. The highest-order coefficient of $\xi w_s$ is $\mathrm{lc}(\xi)\,\mathrm{lc}(w_s)$ and its lowest-order nonzero coefficient is $\mathrm{tc}(\xi)\,\mathrm{tc}(w_s)$. Both $\mathrm{lc}(\xi)$ and $\mathrm{tc}(\xi)$ are nonzero integers, so $\lVert\xi w_s\rVert_\infty\ge\max\{\lvert\mathrm{lc}(w_s)\rvert,\lvert\mathrm{tc}(w_s)\rvert\}$. Taking the maximum over $s$ gives $\lVert v\rVert_\infty\ge\beta(K)$. Moving the support multiplies $\hat v$ by a power of $\zeta$, which changes neither coefficient. Finally, $\hat v=w$ is itself a solution of sup-norm $\lVert w\rVert_\infty$.
\end{proof}

\begin{remark}[Border columns]\label{rem:border}
Within one filter support of the left or right image edge, the resampler clips and renormalizes its weights, so the periodic description and the bound $\beta(K)$ do not apply there. We call an input column \emph{interior} when every output column that reads it has periodic weights. A carrier that touches the border must be checked with Proposition~\ref{prop:construct} on that window. \AF{} places carriers on interior columns.
\end{remark}

\begin{table}[htbp]
\centering
\footnotesize
\caption{\textbf{Window-free bounds in the principal case.} $\beta(K)$ bounds every interior integer kernel vector from below at any support width, $\lVert w\rVert_\infty$ is the sup-norm of the generator, and $\beta(K)/M$ bounds interior image differences of any rank, where $M$ is the largest absolute row sum of $K$. Orange entries exceed $255$ and exclude a realizable perturbation. Blue entries are attained by the generator. The box operator at $1.5\times$ fails the bit-exact reconstruction check of Table~\ref{tab:kernels}.}
\label{tab:beta}
\vspace{1pt}
\setlength{\tabcolsep}{2pt}
\begin{tabular*}{\textwidth}{@{\extracolsep{\fill}}l ccc ccc ccc@{}}
\thead
& \multicolumn{3}{c}{\thd{\boldmath$2\times$}} & \multicolumn{3}{c}{\thd{\boldmath$1.5\times$}} & \multicolumn{3}{c}{\thd{\boldmath$4/3\times$}} \\
\tpartial{\cmidrule(lr){2-4}\cmidrule(lr){5-7}\cmidrule(l){8-10}}
\thd{Filter} & $\beta(K)$ & $\lVert w\rVert_\infty$ & $\beta(K)/M$ & $\beta(K)$ & $\lVert w\rVert_\infty$ & $\beta(K)/M$ & $\beta(K)$ & $\lVert w\rVert_\infty$ & $\beta(K)/M$ \\
\theadend
Bicubic & \tkey{cCert}{$\mathbf{3}$} & \tkey{cCert}{$\mathbf{3}$} & \tmut{$<1$} & \tkey{cExcl}{$2.8{\times}10^{10}$} & $6.7{\times}10^{12}$ & \tkey{cExcl}{$5.7{\times}10^{3}$} & \tkey{cExcl}{$1.2{\times}10^{16}$} & $7.1{\times}10^{18}$ & \tkey{cExcl}{$2.4{\times}10^{9}$} \\
\thairx{1.2pt}
Bilinear & \tkey{cCert}{$\mathbf{3}$} & \tkey{cCert}{$\mathbf{3}$} & \tmut{$<1$} & \tkey{cExcl}{$5.2{\times}10^{12}$} & $5.2{\times}10^{12}$ & \tkey{cExcl}{$1.2{\times}10^{6}$} & \tkey{cExcl}{$2.3{\times}10^{12}$} & $2.3{\times}10^{12}$ & \tkey{cExcl}{$5.5{\times}10^{5}$} \\
\thairx{1.2pt}
Box & \tkey{cCert}{$\mathbf{1}$} & \tkey{cCert}{$\mathbf{1}$} & \tmut{$<1$} & \multicolumn{3}{c}{\tmut{\emph{gate failed}}} & \tkey{cCert}{$\mathbf{1}$} & \tkey{cCert}{$\mathbf{1}$} & \tmut{$<1$} \\
\thairx{1.2pt}
Lanczos & \tkey{cExcl}{$6.3{\times}10^{4}$} & $1.9{\times}10^{6}$ & \tmut{$<1$} & \tkey{cExcl}{$5.6{\times}10^{9}$} & $6.8{\times}10^{12}$ & \tkey{cExcl}{$964$} & \tkey{cExcl}{$1.2{\times}10^{15}$} & $2.3{\times}10^{19}$ & \tkey{cExcl}{$2.1{\times}10^{8}$} \\
\bottomrule
\end{tabular*}
\setlength{\tabcolsep}{6pt}
\end{table}

\paragraph{Computed bounds} Table~\ref{tab:beta} evaluates Theorem~\ref{thm:poly} on the reconstructed operators at the three principal ratios. We compute the maximal minors and their greatest common divisor in exact integer arithmetic and confirm that the operator maps the generator $w$ exactly to zero. In every case the interior columns are exactly periodic and $H$ has full row rank. Where $\beta(K)$ exceeds $255$, no realizable interior kernel vector exists at any support width, and where $\beta(K)/M$ also exceeds $255$, no interior image difference of any rank lies in the exact kernel.

\section{The \AF{} procedure}\label{app:algo}

Algorithm~\ref{alg:af} separates \AF{} into two procedures. \textsc{Audit} runs once per deployed configuration without model access or image content. It accepts the reconstructed operator $K$ only after reproducing the deployed resize bit for bit. When a pipeline applies multiple resizes to the same pixels, it stacks their operators so that the resulting kernel is the intersection of their individual kernels, and it searches the integer kernel with LLL basis reduction \citep{lll}. \textsc{Construct} then builds a pair whose desired fact values are $\ell_0$ and $\ell_1$. A \textsc{not constructible} verdict applies to the rank-one family $uv^{\!\top}$.

\begin{algorithm}[t]
\caption{\AF{} configuration audit and pair construction}
\label{alg:af}
\small
\begin{algorithmic}[1]
\Procedure{Audit}{$\Phi,W$}
  \State Enumerate the image-dependent fields $\mathcal{F}$ emitted by $\Phi$
  \State Perturb each field separately and verify that the digest changes
  \State Reconstruct the integer resize operators from their fixed-point coefficients
  \State Stack operators applied to the same pixels into $K$
  \If{$K$ does not reproduce the deployed resizes bit for bit}
    \State \Return \textsc{audit failure}
  \EndIf
  \State Compute and LLL-reduce an exact $\mathbb{Z}$-basis $B$ of $L=\ker K_W\cap\mathbb{Z}^{W}$
  \If{$L=\{0\}$}
    \State \Return \textsc{not constructible}
  \EndIf
  \State Compute the Gram--Schmidt vectors $b_1^*,\dots,b_n^*$ of $B$
  \If{$\min_i\lVert b_i^*\rVert_2>255\sqrt{W}$}
    \State \Return \textsc{not constructible} \Comment{Proposition~\ref{prop:construct}}
  \EndIf
  \State Search $L$ for $v\neq0$ with $\lVert v\rVert_\infty\le255$
  \If{no such $v$ is found}
    \State \Return \textsc{unresolved}
  \EndIf
  \State \Return $(\mathcal{F},K,v)$
\EndProcedure
\Statex
\Procedure{Construct}{$\bx,\ell_0,\ell_1,\mathcal{F},K,v$}
  \State Choose an integer amplitude $m\ge1$ with $m\lVert v\rVert_\infty\le255$
  \State Place copies of $v$ at valid column offsets to form a carrier $\tilde v$ satisfying $K\tilde v=0$
  \State Choose binary row selectors $u_0$ and $u_1$ that render fact values $\ell_0$ and $\ell_1$ without clipping
  \State Form $\bx_b$ by adding $m\tilde v$ to $\bx$ on the rows selected by $u_b$, for $b\in\{0,1\}$
  \If{$\Phi(\bx_0)$ and $\Phi(\bx_1)$ have equal digests over $\mathcal{F}$}
    \State \Return the certified pair and its certification record
  \EndIf
  \State \Return \textsc{fail}
\EndProcedure
\end{algorithmic}
\end{algorithm}

\paragraph{Construction settings} For the $896\!\to\!448$ bicubic resize, the shortest integer kernel vector is $v=(1,-3,3,-1)$, tiled across valid column offsets to form the full-width carrier $\tilde v$. Every certified pair in Table~\ref{tab:e0} uses amplitude $m=20$, giving a maximum pixel offset of $60$. The tiling pipeline uses a vector of sup-norm $7$, giving a maximum offset of $140$. The fact values are $\ell_0=2$ and $\ell_1=3$ visible horizontal bands, so $\bx_1-\bx_0=m\tilde v$ on the added rows and is zero elsewhere.

\paragraph{Shapes beyond bands} Any image difference $\bd$ whose rows lie in $\ker K$ satisfies $K\bd K^{\!\top}=0$. Figure~\ref{fig:teaser} applies this row-wise construction to digit glyphs that visibly depict a $3$ and an $8$. The resulting difference has rank $18$ and lies outside the rank-one family bounded by Proposition~\ref{prop:construct}.

\section{Interface completeness checklist}\label{app:complete}

For each scored prompt, the certificate hashes every field returned by the deployed processor (Table~\ref{tab:fields}). Arrays are serialized with their field name, data type, byte order, shape and raw bytes, while scalars and sequences use a canonical serialization. Fields are ordered by name before being combined into one SHA-256 digest. Downstream image-dependent quantities, including projected visual tokens, position identifiers and expanded image placeholders, are determined by the hashed bundle and the shared non-image inputs.

\paragraph{Validation} Perturbing one element in each enumerated field changed the digest for all five Qwen2.5-VL fields and all three LLaVA-1.5 fields. Re-hashing five certified pairs per target with the complete scored prompt $(Q,R)$ left every pair identical across all fields.

\begin{table}[htbp]
\centering
\footnotesize
\caption{\textbf{Fields included in the certificate digest.} Shapes are shown for the audited scored prompt. Image dependence indicates whether a field varies with pixel content, only with image geometry, or not with the image.}
\label{tab:fields}
\vspace{1pt}
\newcommand{\fld}[1]{\textsf{#1}}
\begin{tabular*}{\textwidth}{@{\extracolsep{\fill}}lllll@{}}
\thead
& & \multicolumn{2}{c}{\thd{Shape}} & \\
\tpartial{\cmidrule(lr){3-4}}
\thd{Processor field} & \thd{Data type} & \thd{Qwen2.5-VL-7B} & \thd{LLaVA-1.5-7B} & \thd{Image dependence} \\
\theadend
\fld{pixel\_values} & float32 & $(1024,\,1176)$ & $(1,\,3,\,336,\,336)$ & content \\
\fld{image\_grid\_thw} & int64 & $(1,\,3)$ & \tmut{not emitted} & geometry \\
\fld{input\_ids} & int64 & $(1,\,321)$ & $(1,\,617)$ & \tmut{none} \\
\fld{attention\_mask} & int64 & $(1,\,321)$ & $(1,\,617)$ & \tmut{none} \\
\fld{mm\_token\_type\_ids} & int64 & $(1,\,321)$ & \tmut{not emitted} & \tmut{none} \\
\bottomrule
\end{tabular*}
\end{table}

\section{Alias construction in detail}\label{app:construct}

\paragraph{Windowed criterion} We apply Proposition~\ref{prop:construct} to $32$-column windows at stride $16$ and take the minimum bound over the windows. A window of $W$ columns at stride $t$ contains every support of width at most $W-t+1$, so the $32$-column scan certifies supports up to $17$ columns wide and the $256$-column scan up to $241$. Widening the windows to $64$, $128$ or $256$ columns changes no verdict among the $27$ filter-and-ratio configurations that pass the bit-exact reconstruction check. Clearing denominators in a basis of the rational kernel may span only a strict subset of $L$, so we obtain a $\mathbb{Z}$-basis using integer column operations of determinant $\pm1$, shorten it with LLL and evaluate the Gram--Schmidt bound in exact rational arithmetic.

\paragraph{Resize results} We reconstructed $K$ for $28$ resampling-filter and ratio combinations (Table~\ref{tab:kernels}). The reconstruction matched the library bit for bit in $27$ cases, and the criterion returned a definitive verdict in $26$. Unlike bicubic, Lanczos is ruled out through $3.01\times$ but succeeds at $6\times$ and $8\times$. Each positive result is verified by resizing a constructed pair, and Table~\ref{tab:breadth} reports the deployed processors that use other resampling filters. The $4\times$ Lanczos case remains unresolved because the bound falls below $255$ but the reduced basis yields no realizable vector, and the box operator at $1.5\times$ fails the bit-exact reconstruction gate.

\begin{table}[htbp]
\centering
\footnotesize
\caption{\textbf{Constructibility by resampling filter and resize ratio.} Bold numbers give the sup-norm of a constructed integer kernel vector. Entries of the form $\ge B$ give certified lower bounds. Because each such bound exceeds $255$, the rank-one construction is excluded. At $1.5\times$ and for Lanczos at $2\times$, Theorem~\ref{thm:poly} makes the exclusion independent of the window (Table~\ref{tab:beta}). \emph{Unresolved} means that the bound falls below $255$ but no realizable vector was found. \emph{Gate failed} means that the reconstructed operator did not reproduce the deployed library bit for bit.}
\label{tab:kernels}
\vspace{1pt}
\setlength{\tabcolsep}{2pt}
\begin{tabular*}{\textwidth}{@{\extracolsep{\fill}}l@{\hspace{1.2em}} *{7}{>{\centering\arraybackslash}p{0.108\textwidth}}@{}}
\thead
& \multicolumn{7}{c}{\thd{Resize ratio}} \\
\tpartial{\cmidrule(l){2-8}}
\thd{Filter}
& \thd{\boldmath$1.5\times$}
& \thd{\boldmath$2\times$}
& \thd{\boldmath$3\times$}
& \thd{\boldmath$3.01\times$}
& \thd{\boldmath$4\times$}
& \thd{\boldmath$6\times$}
& \thd{\boldmath$8\times$} \\
\theadend
Bicubic
& \makebox[0pt][c]{\tkey{cExcl}{$\ge 1.6{\times}10^{12}$}}
& \tkey{cCert}{$\mathbf{3}$}
& \tkey{cCert}{$\mathbf{13}$}
& \makebox[0pt][c]{\tkey{cExcl}{$\ge 735$}}
& \tkey{cCert}{$\mathbf{5}$}
& \tkey{cCert}{$\mathbf{2}$}
& \tkey{cCert}{$\mathbf{2}$} \\
\thairx{1.2pt}
Bilinear
& \makebox[0pt][c]{\tkey{cExcl}{$\ge 1.2{\times}10^{12}$}}
& \tkey{cCert}{$\mathbf{3}$}
& \tkey{cCert}{$\mathbf{2}$}
& \tkey{cCert}{$\mathbf{1}$}
& \tkey{cCert}{$\mathbf{1}$}
& \tkey{cCert}{$\mathbf{1}$}
& \tkey{cCert}{$\mathbf{1}$} \\
\thairx{1.2pt}
Box
& \tmut{\emph{gate failed}}
& \tkey{cCert}{$\mathbf{1}$}
& \tkey{cCert}{$\mathbf{1}$}
& \tkey{cCert}{$\mathbf{1}$}
& \tkey{cCert}{$\mathbf{1}$}
& \tkey{cCert}{$\mathbf{1}$}
& \tkey{cCert}{$\mathbf{1}$} \\
\thairx{1.2pt}
Lanczos
& \makebox[0pt][c]{\tkey{cExcl}{$\ge 1.8{\times}10^{12}$}}
& \makebox[0pt][c]{\tkey{cExcl}{$\ge 4.8{\times}10^{5}$}}
& \makebox[0pt][c]{\tkey{cExcl}{$\ge 1.6{\times}10^{3}$}}
& \makebox[0pt][c]{\tkey{cExcl}{$\ge 2.0{\times}10^{3}$}}
& \tmut{\emph{unresolved}}
& \tkey{cCert}{$\mathbf{35}$}
& \tkey{cCert}{$\mathbf{16}$} \\
\bottomrule
\end{tabular*}
\setlength{\tabcolsep}{6pt}
\end{table}

\paragraph{Configuration screen} Table~\ref{tab:screen} applies the criterion to $20$ declared bicubic-resize configurations, each a released processor's resize at a stated input and output size. Configurations with the same resampling filter and input-output geometry share the same integer operator. Seventeen configurations across four resize ratios are constructible, and three across two ratios are excluded within the rank-one family. Four cases use declared geometries that differ from the processors' default behavior (Table~\ref{tab:breadth}), so their verdicts should not be interpreted as certificates for the default pipelines.\looseness=-1

\begin{table}[htbp]
\centering
\footnotesize
\caption{\textbf{Screen of $20$ declared bicubic-resize configurations.} The bound column gives the lower bound of Proposition~\ref{prop:construct} on $\lVert v\rVert_\infty$, and a value above $255$ excludes the construction, and at $1.5\times$ the exclusion holds for every interior support (Theorem~\ref{thm:poly}). For constructible ratios, the witness column gives the smallest $\lVert v\rVert_\infty$ of a kernel vector found. A dagger marks a configuration evaluated at a declared geometry that differs from the processor's default, so its verdict does not apply to the default pipeline.}
\label{tab:screen}
\vspace{2pt}
\newcommand{\cfg}[2]{\mbox{#1\,{\scriptsize\color{cPend}#2}}}
\newcommand{\dg}{\smash{\textcolor{cInk}{$^\dagger$}}}
\newcommand{\sep}{\hskip0.4em\cleaders\hbox to 0.3em{\hss\textcolor{cPend}{$\cdot$}\hss}\hskip0.3em\hskip0.4em plus 0.2em\relax}
\setlength{\tabcolsep}{3pt}
\begin{tabular*}{\textwidth}{@{} r @{\hspace{1.0em}} c @{\hspace{1.0em}} c @{\extracolsep{\fill}\hspace{1.2em}} >{\raggedright\arraybackslash}p{0.70\textwidth} @{}}
\thead
\multicolumn{1}{@{}l@{\hspace{1.0em}}}{\thd{Ratio}}
& \thd{Bound}
& \thd{Witness}
& \thd{Configurations}\thdsub{\;\,model and output side in pixels} \\
\theadend
\tgroupcert{4}{Constructible}{integer ratios, $17$ configurations}
$2\times$ & \tmut{$0.71$} & \tkey{cCert}{$\phantom{1}3$}
& \cfg{CLIP}{224}\sep\cfg{PaliGemma}{224}\sep\cfg{Qwen2.5-VL}{224}\sep\cfg{LLaVA-1.5}{336}\sep\cfg{Phi-3.5-V}{336\dg}\sep\cfg{SigLIP}{384}\sep\cfg{SmolVLM}{384\dg}\sep\cfg{InternVL thumbnail}{448}\sep\cfg{Qwen2.5-VL}{448}\sep\cfg{Pixtral}{512}\sep\cfg{Cambrian}{576} \\
\tgreyrule{1.6pt}
$3\times$ & \tmut{$4.17$} & \tkey{cCert}{$13$}
& \cfg{BLIP-2}{224}\sep\cfg{SigLIP}{384}\sep\cfg{InternVL tile}{448} \\
\tgreyrule{1.6pt}
$4\times$ & \tmut{$1.25$} & \tkey{cCert}{$\phantom{1}5$}
& \cfg{CLIP}{224}\sep\cfg{LLaVA-1.5}{336} \\
\tgreyrule{1.6pt}
$6\times$ & \tmut{$0.59$} & \tkey{cCert}{$\phantom{1}2$}
& \cfg{Molmo}{336} \\
\tgroupsep
\tgroupexcl{4}{Excluded}{non-integer ratios, $3$ configurations}
$1.5\times$ & \tkey{cExcl}{$1.6{\times}10^{12}$} & \tmut{none}
& \cfg{Fuyu}{300\dg}\sep\cfg{Idefics2}{378\dg} \\
\tgreyrule{1.6pt}
$2.5\times$ & \tkey{cExcl}{$1.8{\times}10^{4\phantom{0}}$} & \tmut{none}
& \cfg{MiniCPM-V}{448} \\
\bottomrule
\end{tabular*}
\setlength{\tabcolsep}{6pt}
\end{table}

\paragraph{Where collisions can be constructed} Whether \AF{} can construct a collision at a preprocessing stage depends on the deployed configuration, which we determine from its implementation. A collision at resizing identifies information lost before the learned encoder, whereas a collision at token selection would identify information lost by the encoder. We exclude post-encoder pruning because global attention allows retained token values to depend on discarded content, violating Assumption~\ref{as:factor}. Across the three target pipelines, only resizing produced certified pairs. Cropping and tiling produced no candidates because the tested pipelines preserve all image regions, and we did not evaluate pre-encoder masking.

\paragraph{Deployment dependence} A resize without downscaling is an identity map and admits no construction. One target resizes only when the image exceeds a pixel budget \citep{qwen2vl}, so an aligned image within that budget passes through unchanged, and lowering the budget introduces downscaling. If the pipeline also appends a thumbnail \citep{internvl15,llavaonevision}, a perturbation must be removed by both resize operations.

\paragraph{Input transport} A certificate applies to the exact $8$-bit pixel arrays received by the processor. Lossy re-encoding during transport, such as JPEG compression, invalidates it unless equality is verified again after decoding. The audited Qwen and LLaVA processors resize $8$-bit values before rescaling or normalization, so no floating-point transformation precedes the certified collision.

\paragraph{Natural images} On a natural image, every affected pixel must remain within $[0,255]$ after adding the carrier. A solid band across the image therefore certified none of the $200$ benchmark images at any tested amplitude. Because headroom is required only on affected rows, we retain eligible rows and allow the bands, which we call marks there, to vary in height, on benchmark images resized to $896\times896$ for the deployed $2\times$ regime. None of the $200$ images has a native size at which the processor's scale factor is an integer. At amplitude $20$, $48$ of $200$ images certify with all three marks aligned in one column, and $94$ certify at amplitude $5$. Requiring each mark to occupy a flat, mid-tone region leaves $8$ legible certified pairs at amplitude $20$. Padding an image onto the required canvas preserves every original pixel and yields $13$ certified pairs, of which $5$ are legible (Figure~\ref{fig:pair}).

\paragraph{Libraries and numeric modes} We resized five bicubic-carrier and five box-carrier pairs at $2\times$ under the resize modes of Pillow and torchvision and compared the outputs exactly. The bicubic carrier $(1,-3,3,-1)$ collides bit for bit under antialiased bicubic and bilinear resampling in both libraries, in 8-bit fixed-point and 32-bit floating-point arithmetic alike, because the $2\times$ weights are exactly representable. The box carrier $(1,-1)$ collides under box, area and non-antialiased bilinear resampling. Neither collides under any other tested mode, including Lanczos. A pair is therefore specific to one operator, whatever the library or number format, while Theorem~\ref{thm:poly} gives every such downscaler kernel vectors of its own. These are the two image backends the audited processors offer. No audited pipeline uses OpenCV or TensorFlow, which we did not evaluate.

\paragraph{Breadth across deployed processors} We applied the same construction to every pipeline whose processor was available (Table~\ref{tab:breadth}). Across four additional pipelines using a single fixed-point resize at an eligible integer ratio, all six attempts per pipeline certified, for $24/24$ overall. The current procedure does not resolve tiled policies, patch-based geometries or floating-point resizing. One tiled pipeline first resizes at $4/3\times$ with a Lanczos filter. Under the fixed-point resampler, $\beta(K)=1.2\times10^{15}$ there, so that stage admits no realizable exact-kernel perturbation on interior support at any rank (Table~\ref{tab:beta}).

\begin{table}[htbp]
\centering
\footnotesize
\caption{\textbf{Breadth across deployed processors.} All six attempts certified on each pipeline with an eligible fixed-point resize. The remaining pipelines fall outside the current procedure's scope.}
\label{tab:breadth}
\vspace{1pt}
\begin{tabular*}{\textwidth}{@{\extracolsep{\fill}}l r @{\extracolsep{0pt}\hspace{1.8em}} r @{\hspace{1.8em}} r @{\extracolsep{\fill}} l@{}}
\thead
& \multicolumn{3}{c}{\thd{Observed resize}} & \\
\tpartial{\cmidrule(lr){2-4}}
\thd{Pipeline} & \thd{Input side} & \thd{Output side} & \thd{Ratio} & \thd{Result} \\
\theadend
\tgroupcert{5}{Eligible fixed-point resizes}{}
SigLIP-B/16-384 & $768$ & $384$ & $2\times$ & \tkey{cCert}{$6/6$ certified} \\
SigLIP-SO400M/14-384 & $768$ & $384$ & $2\times$ & \tkey{cCert}{$6/6$ certified} \\
BLIP-2 OPT-2.7B & $448$ & $224$ & $2\times$ & \tkey{cCert}{$6/6$ certified} \\
Pixtral-12B & $2048$ & $1024$ & $2\times$ & \tkey{cCert}{$6/6$ certified} \\
\tgroupsep
\tgroup{5}{Outside the current procedure's scope}{}
SmolVLM-Instruct & $2048$ & $1536$ & $1.33\times$ & \tmut{tiled policy} \\
Idefics2-8B & $2048$ & $1960$ & $1.04\times$ & \tmut{tiled policy} \\
Qwen2-VL-7B-Instruct & $512$ & $504$ & $1.02\times$ & \tmut{near-identity policy} \\
Phi-3.5-Vision-Instruct & $1344$ & $672$ & $2\times$ & \tmut{floating-point resize} \\
Fuyu-8B & \multicolumn{3}{c}{\tmut{no resize}} & \tmut{patch geometry} \\
\bottomrule
\end{tabular*}
\end{table}

\section{Instrument validation}\label{app:e0}

Table~\ref{tab:e0} summarizes the end-to-end validation. On Qwen2.5-VL-7B, all $15$ constructed pairs certified under a $2\times$ downscale, while none of the $15$ matched non-kernel controls, the $15$ identity-resize controls or the $15$ non-kernel controls under identity resize did. The identity-resize controls keep the stock pixel budget, so nothing is downscaled. For the $8$ pairs with available logits, member differences at the decision position were exactly zero and all matched controls were nonzero. Without adapting the construction, all $10$ LLaVA-1.5-7B pairs certified with zero logit difference and all $10$ controls separated. VisualPRM-8B processes a tiled view and a thumbnail, so the construction uses the intersection of the two resize kernels. All $20$ pairs collided across the complete interface and all $20$ controls separated. We omit its logit gap because the checkpoint returned non-finite language-model logits under every tested hardware and precision configuration. Across $72$ rounded real-valued null-space candidates spanning downscale factors from $1.41\times$ to $8\times$, roughly $1\%$ of the $602{,}112$ resized values remained unequal, about $6000$ mismatches per tensor, and no candidate certified. In total, the $147$ negative controls comprise $45$ non-kernel patterns across the three targets, $30$ identity-resize attempts and these $72$ rounded candidates.

On all $80$ certified pairs of the routing screen, a fixed deterministic critique produced pair-balanced accuracy of exactly $1/2$. Re-attention and further reasoning use independently sampled critiques, so their finite-sample outcomes can vary around $1/2$ even though Theorem~\ref{thm:cert} guarantees identical output laws. Re-encoding lies outside the language-side class, and its gain was positive on $13$ pairs, zero on $60$ and negative on $7$. Near-chance accuracy alone is not evidence of a collision. The identity-resize controls have pair-balanced accuracies from $0.4975$ to $0.5028$ but nonzero member logit differences.

\section{Actions and cost accounting}\label{app:actions}

\begin{table}
\centering
\fontsize{7.4}{8.6}\selectfont
\caption{\textbf{Action set and diagnostic controls.} The policy actions use the path-matched costs shown below. The placebo critique is the path-matched baseline, the ten-crop oracle bounds crop-based re-encoding, and the same-interface re-encode isolates re-encoding overhead.}
\label{tab:actions}
\vspace{1pt}
\setlength{\tabcolsep}{3pt}
\newcommand{\act}[2]{\textsf{#1}\hspace{0.5em}\tmut{#2}}
\begin{tabular*}{\textwidth}{@{\extracolsep{\fill}}l l c l@{}}
\thead
\thd{Action} & \thd{Mechanism} & \thd{Changes $\bz$} & \thd{Cost accounting} \\
\theadend
\tgroup{4}{Policy actions}{}
\textsf{stop}
  & accept the current verdict
  & \tmut{no}
  & $0$ \\
\act{att}{re-attend}
  & instruct the model to re-examine the region named by $Q$
  & \tmut{no}
  & prompt + one critique \\
\act{rea}{reason}
  & run and aggregate $t$ additional critique passes
  & \tmut{no}
  & $t\times$ critique \\
\act{enc}{re-encode}
  & recompute $\bz$ from a higher-resolution center crop
  & yes
  & re-prefill + one critique \\
\tgroupsep
\tgroup{4}{Diagnostic controls}{}
Placebo critique
  & the \textsf{att} path without the instruction to re-examine
  & \tmut{no}
  & prompt + one critique \\
Ten-crop oracle
  & re-encode the best of ten half-frame crops
  & yes
  & matched to \textsf{enc} \\
Same-interface re-encode
  & re-encode while reproducing the original interface state
  & \tmut{no}
  & matched to \textsf{enc} \\
\bottomrule
\end{tabular*}
\setlength{\tabcolsep}{6pt}
\end{table}

\paragraph{Action costs} Table~\ref{tab:actions} specifies the four actions. Normalized to one critique and net of stopping, the decode-token and prefill-token cost vectors are $(0,\,1.00,\,2.75,\,1.01)$ and $(0,\,1.00,\,2.93,\,1.02)$ for $(\textsf{stop},\textsf{att},\textsf{rea},\textsf{enc})$, and wall-clock time gives $(0,\,1.00,\,3.51,\,1.00)$. The $\lambda$ sweep of Section~\ref{sec:ias} uses decode-token cost. The $\textsf{enc}$ action uses a question-independent crop from the central half of the image. A same-interface re-encode produces an identical digest, hence an identical output law and exactly zero gain on all $1280$ items. Question-conditioned acquisition policies \citep{vstar,tikart} could perform better, so the measured $\Delta(\textsf{enc})$ does not upper-bound re-encoding in general.

\paragraph{Labels} Natural-item labels are the step-correctness annotations released with \citet{visualprm}. Certified-pair labels need no adjudication because the construction writes the two fact values into known image regions.

\section{Experimental protocol}\label{app:protocol}

\paragraph{Routing rules and controls} At $\lambda=0$, the specified \IAS{} rule selects reasoning on $1266$ of $1280$ natural items and re-encoding on the remaining $14$. At $\lambda=0.01$, it stops on $1184$ natural items and every certified pair. Across the full $\lambda$ sweep, every protocol-specified policy remains within $0.002$ regret of stopping, and all adaptive rules reduce to stopping from $\lambda=0.05$ onward under each cost measure. The specified rule requires no labels or measured gains from test items. The median-threshold diagnostic acts on the items whose score is at or below its median and stops on the rest. It takes the non-stop action with the largest gain on the choosing replicate, and the random score controls for this outcome-informed selection. Raw interface distance is the root-mean-square magnitude of the pixel tensor. Besides the routers of Table~\ref{tab:routers}, two interface-only routers use a $1024$-dimensional size-invariant projection of the pixel tensor, and neither improves over stopping.

\paragraph{Statistical plan} All bootstrap intervals are clustered over the $1018$ source records containing the $1280$ items. Each item is measured in two replicates. Across replicates, verdict probabilities correlate between $0.84$ and $0.94$, path-matched gains between $0.47$ and $0.72$, and the per-item optimal action agrees on $65.9\%$ of items. Reversing the choosing and scoring replicates changes every regret in Table~\ref{tab:main} by at most $0.004$.\looseness=-1

\paragraph{Decoder} The decoder $d_\theta$ is an $\ell_2$-regularized logistic regression on a $128$-dimensional Gaussian random projection of the flattened interface tensor. No hyperparameters are tuned. We use five-fold cross-fitting, with folds grouped by source image for rendered items and by record for the natural stream to prevent leakage between related items. Confidence is calibrated out of fold using monotone isotonic regression. The question and candidate step define the fact being decoded but are not decoder features.

\section{Power calculation for the routing gate}\label{app:power}

We estimate power by simulation because $V$ maximizes over actions and items are clustered within templates, the planned groups of eight related items. Each simulated template receives a best action. An item inherits it with probability $\rho$ and receives an independently drawn action otherwise (Table~\ref{tab:power}). No false positives occurred in the null simulations over this grid.

\begin{table}[htbp]
\centering
\footnotesize
\caption{\textbf{Power of the routing-value gate.} Each cell gives the simulated probability that the $95\%$ lower bound on the split-sample routing value $V$ exceeds zero. The columns give the number of eight-item templates, and $\rho$ controls dependence among items from the same template. Bold values mark the smallest template count achieving at least $0.90$ power for the pre-registered target $V=0.03$.}
\label{tab:power}
\vspace{1pt}
\setlength{\tabcolsep}{2pt}
\begin{tabular}{@{}*{2}{Y{0.105\textwidth}}@{\hspace{2.2em}}*{6}{Y{0.072\textwidth}}@{}}
\theadn{2}{3.6pt}
\multicolumn{2}{@{}c@{\hspace{2.2em}}}{\thd{Simulation setting}}
& \multicolumn{6}{c@{}}{\thd{Number of eight-item templates}} \\
\tpartial{\cmidrule(r{2.2em}){1-2}\cmidrule{3-8}}
\thd{\boldmath$V$} & \thd{\boldmath$\rho$}
& \thd{\boldmath$10$} & \thd{\boldmath$20$} & \thd{\boldmath$40$}
& \thd{\boldmath$80$} & \thd{\boldmath$160$} & \thd{\boldmath$320$} \\
\theadend
$0.01$ & $0.5$ & $0.03$ & $0.07$ & $0.12$ & $0.27$ & $0.45$ & $0.71$ \\
\thairx{0.7pt}
$0.02$ & $0.5$ & $0.08$ & $0.16$ & $0.35$ & $0.68$ & $0.94$ & $0.99$ \\
\thairx{0.7pt}
$0.03$ & $0.5$ & $0.13$ & $0.33$ & $0.68$ & \tkey{cCert}{$\mathbf{0.93}$} & $1.00$ & $1.00$ \\
\thairx{0.7pt}
$0.03$ & $0.9$ & $0.03$ & $0.15$ & $0.52$ & $0.86$ & \tkey{cCert}{$\mathbf{0.99}$} & $1.00$ \\
\thairx{0.7pt}
$0.05$ & $0.5$ & $0.34$ & $0.77$ & $0.96$ & $1.00$ & $1.00$ & $1.00$ \\
\bottomrule
\end{tabular}
\setlength{\tabcolsep}{6pt}
\end{table}

We power the gate for $V=0.03$, the smallest value that remains positive under the cost accounting. At this effect size, $160$ eight-item templates give $0.99$ power at $\rho=0.9$, so the design uses $1280$ items. The realized stream has $1018$ records and an estimated within-record correlation of $0.25$, which gives power $1$ at $V=0.03$ and $0.92$--$0.98$ at $V=0.02$.

\section{Estimating the routing value}\label{app:bookkeeping}

\paragraph{Bias of the plug-in estimate} A plug-in estimator replaces true gains with measured gains. Taking the maximum within each item then favors actions with positive errors, and more items do not remove this bias. In a simulation with no item-specific signal, the plug-in estimate was $0.125$ with $80$ items and $0.128$ with $1200$, and its bootstrap interval excluded zero in every simulation run. The simulated noise scale was not fitted to the observed data, so these values should not be compared with the $0.101$ in Table~\ref{tab:ladder}. We therefore select per-item maxima on one replicate and evaluate them on the other (Table~\ref{tab:ladder}). Sampling is needed because $\textsf{rea}$ aggregates multiple critiques, and verdicts read directly from decision-position logits would make the replicates identical.

\begin{table}[htbp] \centering \footnotesize
\caption{\textbf{From estimated routing headroom to realized gain} on the natural stream. Every entry is an estimate.}
\label{tab:ladder}
\begin{tabular*}{\textwidth}{@{\hspace{\tinset}\extracolsep{\fill}}l l r >{\raggedright\arraybackslash}p{5.5cm}@{\hspace{\tinset}}}
\thead
\thd{Symbol} & \thd{Quantity} & \thd{Value} & \thd{Interpretation} \\
\theadend
$\hat V_{\text{plug}}$ & plug-in per-item maximum & $0.101$ & inflated by maximizing noisy gains \\
\thairx{1.2pt}
$\hat V$ & split-sample routing value & $0.056$ & estimated headroom after bias control \\
\thairx{1.2pt}
$\hat V_{\text{fit}}$ & best of eight fitted interface routers & $\approx0$ & no fitted router recovers the headroom \\
\thairx{1.2pt}
$\hat V_{\IAS}$ & \ROUTE{} as realized & $0$ & matches the best fixed action \\
\bottomrule
\end{tabular*}
\end{table}

The selected action and its regret are unchanged when a per-item constant is added to every estimated gain, so we report action-ordering accuracy and regret.

\section{Extended results}\label{app:extra}

\paragraph{Unprompted baseline and the crop arm} Relative to the model's unprompted verdict, re-attention, reasoning and re-encoding have gains of $-0.069$, $-0.070$ and $-0.092$. Their difference from the path-matched results in Section~\ref{sec:results} is the prompt path itself, which has a gain of $-0.070$. This per-item offset is identical across replicates, so it leaves action ordering and regret unchanged. A ten-crop oracle over the half-frame center crop and a $3\times3$ grid of half-frame crops upper-bounds any policy restricted to those crops \citep{vicrop}. When the crop is selected on one replicate and scored on the other, the oracle gains $+0.035$ over the path-matched stopping baseline, with a $95\%$ interval of $[+0.022,+0.047]$, but has a gain of $-0.035$ relative to the unprompted verdict.

\begin{table}[htbp]
\centering
\scriptsize
\caption{\textbf{Oracle-supervised routers by signal.} The vision-tower routers read $\bz$ alone, and the hidden-state routers read $(\bz,Q,R)$. Gain intervals are clustered by record, and \emph{Action rate} is the fraction of natural items sent to a non-stop action. On certified pairs, \emph{Any action} counts non-stop choices, while \emph{Dominated} counts choices of re-attention or further reasoning.}
\label{tab:routers}
\vspace{1pt}
\setlength{\tabcolsep}{3.5pt}
\begin{tabular*}{\textwidth}{@{\hspace{\tinset}\extracolsep{\fill}}lllrrrr@{\hspace{\tinset}}}
\theadn{2}{3.2pt}
& & \multicolumn{3}{c}{\thd{Natural stream}}
& \multicolumn{2}{c@{\hspace{\tinset}}}{\thd{Certified pairs}} \\
\tpartial{\cmidrule(lr){3-5}\cmidrule(l){6-7}}
\thd{Features}
& \thd{Learner}
& \thd{Gain over \textsf{stop}, $95\%$ CI}
& \thd{Order acc.}
& \thd{Action rate}
& \thd{Any action}
& \thd{Dominated} \\
\theadend
\tgroup{7}{Vision-tower outputs, a signal of $\bz$}{}
projected tokens
& ridge
& $-0.003$\ \tmut{$[-0.011,\,+0.006]$}
& $0.244$
& $0.77$
& $0/80$
& \tkey{cCert}{$0/80$} \\

projected tokens
& boosting
& $-0.003$\ \tmut{$[-0.012,\,+0.007]$}
& $0.254$
& $0.78$
& $71/80$
& \tkey{cExcl}{$37/80$} \\
\thair

patch features
& ridge
& $+0.002$\ \tmut{$[-0.006,\,+0.010]$}
& $0.245$
& $0.77$
& $80/80$
& \tkey{cExcl}{$80/80$} \\

patch features
& boosting
& $+0.003$\ \tmut{$[-0.006,\,+0.012]$}
& $0.259$
& $0.78$
& $70/80$
& \tkey{cExcl}{$5/80$} \\
\thair

both
& ridge
& $+0.000$\ \tmut{$[-0.008,\,+0.009]$}
& $0.247$
& $0.76$
& $6/80$
& \tkey{cExcl}{$6/80$} \\

both
& boosting
& $-0.002$\ \tmut{$[-0.011,\,+0.007]$}
& $0.248$
& $0.80$
& $62/80$
& \tkey{cExcl}{$62/80$} \\
\tgroupsep
\tgroup{7}{Verifier hidden state at the decision position, a signal of $(\bz,Q,R)$}{}
final layer
& ridge
& $+0.008$\ \tmut{$[+0.000,\,+0.016]$}
& $0.270$
& $0.79$
& $0/80$
& \tkey{cCert}{$0/80$} \\

final layer
& boosting
& $+0.002$\ \tmut{$[-0.007,\,+0.011]$}
& $0.258$
& $0.80$
& $0/80$
& \tkey{cCert}{$0/80$} \\
\thair

middle layer
& ridge
& $+0.008$\ \tmut{$[-0.001,\,+0.016]$}
& $0.282$
& $0.80$
& $0/80$
& \tkey{cCert}{$0/80$} \\

middle layer
& boosting
& $+0.005$\ \tmut{$[-0.004,\,+0.013]$}
& $0.268$
& $0.80$
& $19/80$
& \tkey{cExcl}{$19/80$} \\
\thair

both
& ridge
& $+0.010$\ \tmut{$[+0.002,\,+0.018]$}
& $0.283$
& $0.79$
& $0/80$
& \tkey{cCert}{$0/80$} \\

both
& boosting
& $+0.004$\ \tmut{$[-0.005,\,+0.012]$}
& $0.260$
& $0.81$
& $0/80$
& \tkey{cCert}{$0/80$} \\
\bottomrule
\end{tabular*}
\setlength{\tabcolsep}{6pt}
\end{table}

\begin{table}[t]
\centering
\scriptsize
\caption{\textbf{Results by source benchmark.} Policies are fit once on the full stream and evaluated within each source. Gains are relative to the source-specific best fixed action. Here $\hat V$ is the split-sample routing value, and intervals are $95\%$ confidence intervals clustered by record. The learned router is boosting on patch features. Median-rule entries report \IAS{} and random-score gains, respectively.}
\label{tab:sources}
\vspace{1pt}
\setlength{\tabcolsep}{3.5pt}
\resizebox{\textwidth}{!}{%
\begin{tabular}{@{\hspace{\tinset}}llrrcr@{\hspace{\tinset}}}
\thead
\thd{Source}
& \thd{Best fixed}
& \multicolumn{1}{c}{\thd{$\hat V$}}
& \multicolumn{1}{c}{\thd{\ROUTE{} specified}}
& \multicolumn{1}{c}{\thd{Median \ROUTE{} / random}}
& \multicolumn{1}{c@{\hspace{\tinset}}}{\thd{Learned router}} \\
\theadend
MathVerse ($454$)
& \textsf{att}
& $+0.067$ {\scriptsize\tmut{$[+0.051,+0.078]$}}
& $+0.002$ {\scriptsize\tmut{$[-0.012,+0.015]$}}
& $+0.020$ / $+0.023$
& $-0.004$ {\scriptsize\tmut{$[-0.018,+0.011]$}} \\
\thairx{1pt}
MathVision ($331$)
& \textsf{stop}
& $+0.031$ {\scriptsize\tmut{$[+0.013,+0.042]$}}
& $-0.012$ {\scriptsize\tmut{$[-0.028,+0.003]$}}
& $+0.000$ / $+0.002$
& $-0.007$ {\scriptsize\tmut{$[-0.021,+0.006]$}} \\
\thairx{1pt}
DynaMath ($260$)
& \textsf{stop}
& $+0.065$ {\scriptsize\tmut{$[+0.044,+0.079]$}}
& $-0.004$ {\scriptsize\tmut{$[-0.022,+0.017]$}}
& $+0.019$ / $+0.018$
& $+0.009$ {\scriptsize\tmut{$[-0.013,+0.031]$}} \\
\thairx{1pt}
WeMath ($126$)
& \textsf{rea}
& $+0.044$ {\scriptsize\tmut{$[+0.018,+0.055]$}}
& $+0.001$ {\scriptsize\tmut{$[-0.001,+0.005]$}}
& $+0.019$ / $+0.025$
& $+0.006$ {\scriptsize\tmut{$[-0.012,+0.024]$}} \\
\thairx{1pt}
MMMU ($109$)
& \textsf{rea}
& $+0.034$ {\scriptsize\tmut{$[+0.011,+0.055]$}}
& $-0.001$ {\scriptsize\tmut{$[-0.003,+0.001]$}}
& $-0.006$ / $-0.009$
& $-0.004$ {\scriptsize\tmut{$[-0.022,+0.012]$}} \\
\tunder{cBand}\textbf{All five} ($1280$)
& \textsf{rea}
& $+0.056$ {\scriptsize\tmut{$[+0.047,+0.062]$}}
& $+0.000$ {\scriptsize\tmut{$[+0.000,+0.001]$}}
& $+0.016$ / $+0.018$
& $+0.002$ {\scriptsize\tmut{$[-0.005,+0.010]$}} \\
\tunderend
\bottomrule
\end{tabular}%
}
\setlength{\tabcolsep}{6pt}
\end{table}

\begin{table}[t]
\centering
\footnotesize
\renewcommand{\arraystretch}{0.92}
\caption{\textbf{Natural verdicts under each action.} Columns split the $2560$ measurements by the initial stopping verdict. Proportions use Wilson intervals, and gain intervals are bootstrapped by record.}
\label{tab:a1}
\vspace{1pt}
\begin{tabular*}{\textwidth}{@{\hspace{\tinset}\extracolsep{\fill}}lccc@{\hspace{\tinset}}}
\theadn{2}{3.6pt}
&
&
\multicolumn{2}{c@{\hspace{\tinset}}}{\thd{Initial \textsf{stop} verdict}} \\
\tpartial{\cmidrule(l){3-4}}
\thd{Quantity}
& \thd{All}
& \thd{Wrong}
& \thd{Correct} \\
\theadend
Measurements, $n$
& $2560$
& $954$
& $1606$ \\
\tgroupsep
\tgroup{4}{Inert across three critiques}{}
Rate
& $0.452$ {\scriptsize\tmut{$[0.433,\,0.472]$}}
& $0.496$ {\scriptsize\tmut{$[0.464,\,0.527]$}}
& $0.426$ {\scriptsize\tmut{$[0.403,\,0.451]$}} \\
\tgroupsep
\tgroup{4}{Correct after action}{}
Placebo
& $0.525$ {\scriptsize\tmut{$[0.506,\,0.544]$}}
& $0.126$ {\scriptsize\tmut{$[0.106,\,0.148]$}}
& $0.762$ {\scriptsize\tmut{$[0.741,\,0.782]$}} \\
\textsf{att}
& $0.528$ {\scriptsize\tmut{$[0.509,\,0.547]$}}
& $0.187$ {\scriptsize\tmut{$[0.163,\,0.213]$}}
& $0.731$ {\scriptsize\tmut{$[0.709,\,0.752]$}} \\
\textsf{rea}
& $0.523$ {\scriptsize\tmut{$[0.504,\,0.542]$}}
& $0.152$ {\scriptsize\tmut{$[0.131,\,0.176]$}}
& $0.744$ {\scriptsize\tmut{$[0.722,\,0.764]$}} \\
\textsf{enc}
& $0.500$ {\scriptsize\tmut{$[0.480,\,0.519]$}}
& $0.210$ {\scriptsize\tmut{$[0.185,\,0.237]$}}
& $0.672$ {\scriptsize\tmut{$[0.649,\,0.694]$}} \\
\tgroupsep
\tgroup{4}{Mean gain relative to placebo}{}
\textsf{att}
& $+0.001$ {\scriptsize\tmut{$[-0.007,\,+0.009]$}}
& $+0.030$ {\scriptsize\tmut{$[+0.018,\,+0.041]$}}
& $-0.016$ {\scriptsize\tmut{$[-0.026,\,-0.006]$}} \\
\textsf{rea}
& $+0.003$ {\scriptsize\tmut{$[-0.004,\,+0.010]$}}
& $+0.014$ {\scriptsize\tmut{$[+0.004,\,+0.024]$}}
& $-0.003$ {\scriptsize\tmut{$[-0.013,\,+0.006]$}} \\
\textsf{enc}
& $-0.021$ {\scriptsize\tmut{$[-0.034,\,-0.009]$}}
& $+0.042$ {\scriptsize\tmut{$[+0.026,\,+0.059]$}}
& $-0.059$ {\scriptsize\tmut{$[-0.075,\,-0.043]$}} \\
\bottomrule
\end{tabular*}
\end{table}

\paragraph{Oracle-supervised vision-tower routers} The first block of Table~\ref{tab:routers} reports six routers trained directly on measured action gains with ridge regression or gradient-boosted decision trees. Projected visual tokens delivered to the language model are summarized by their mean and maximum ($7168$ dimensions), and pre-projection patch features by their mean and standard deviation ($2560$ dimensions). The routers are cross-fitted by record and scored on the held-out replicate. For the certified-pair screen, each is refit on the full natural stream, which also supplies the reported natural action rate. Table~\ref{tab:main} reports the lowest-regret router of each block, boosting on patch features and ridge on both hidden layers.

\paragraph{Oracle-supervised hidden-state routers} The second block repeats the experiment with the richest signal the frozen verifier computes from $(\bz,Q,R)$, its language-model hidden state at the decision position, taken at the final layer, a middle layer or both ($3584$ dimensions each). These states are bit-identical across all $80$ certified pairs. The best interval, $[+0.002,+0.018]$, excludes zero with an unadjusted bootstrap $p=0.019$, which Holm correction across the six routers raises to $0.11$. Reversing the choosing and scoring replicates gives the same gain and interval. After correction the gain is not significant, so we do not claim that the text-conditioned signal improves routing.

\paragraph{Source benchmarks} In Table~\ref{tab:sources}, the routing-value interval excludes zero in all five source benchmarks, while the intervals for the specified \IAS{} rule and the learned router contain zero throughout. Under the median rule, the gains from \IAS{} and a random score differ by at most $0.006$.

\paragraph{Natural verifier errors under each action} Table~\ref{tab:a1} divides the $2560$ item-by-replicate measurements into $954$ initially incorrect and $1606$ initially correct stopping verdicts. The placebo critique, the path-matched baseline of Section~\ref{sec:ias}, follows the same prompt path without instructing the model to re-examine the image. \emph{Inert} means that all three sampled critiques produce the same verdict. \emph{Correct after} is the fraction of verdicts that are correct after an action, and \emph{Gain} is the mean path-matched change relative to the placebo. About half of the initial errors are inert, which does not establish that interface states collide (Proposition~\ref{prop:sharp}). Where verdicts change, each action repairs more errors than the placebo but also breaks more initially correct verdicts. These effects cancel over the stream for re-attention and reasoning and are net negative for re-encoding.

\end{document}

%% file: affiliations.tex
\DeclareAffiliation{hbku}{%
  College of Science and Engineering, Hamad Bin Khalifa University, Doha, Qatar}

\DeclareAffiliation{tamu}{%
  Department of Computer and Electrical Engineering,
  Texas A\&M University, College Station, TX, USA}

\DeclareAffiliation{iub}{%
  Luddy School of Informatics, Computing, and Engineering,
  Indiana University Bloomington, Bloomington, IN, USA}

\DeclareAffiliation{uis}{%
  Department of Computer Science,
  University of Illinois Springfield, Springfield, IL, USA}


%% file: references.bib
@misc{visualprm,
      title={{VisualPRM}: An Effective Process Reward Model for Multimodal Reasoning}, 
      author={Weiyun Wang and Zhangwei Gao and Lianjie Chen and Zhe Chen and Jinguo Zhu and Xiangyu Zhao and Yangzhou Liu and Yue Cao and Shenglong Ye and Xizhou Zhu and Lewei Lu and Haodong Duan and Yu Qiao and Jifeng Dai and Wenhai Wang},
      year={2025},
      eprint={2503.10291},
      archivePrefix={arXiv},
      primaryClass={cs.CV},
      url={https://arxiv.org/abs/2503.10291}, 
}

@inproceedings{prmbench,
    title = "{PRMB}ench: A Fine-grained and Challenging Benchmark for Process-Level Reward Models",
    author = "Song, Mingyang  and
      Su, Zhaochen  and
      Qu, Xiaoye  and
      Zhou, Jiawei  and
      Cheng, Yu",
    editor = "Che, Wanxiang  and
      Nabende, Joyce  and
      Shutova, Ekaterina  and
      Pilehvar, Mohammad Taher",
    booktitle = "Proceedings of the 63rd Annual Meeting of the Association for Computational Linguistics (Volume 1: Long Papers)",
    month = jul,
    year = "2025",
    address = "Vienna, Austria",
    publisher = "Association for Computational Linguistics",
    url = "https://aclanthology.org/2025.acl-long.1230/",
    doi = "10.18653/v1/2025.acl-long.1230",
    pages = "25299--25346",
    ISBN = "979-8-89176-251-0"
}

@misc{failtosee,
      title={Failing to See or Failing to Know? {A}ttributing Errors in Vision-Language Models}, 
      author={Khang Nhat Hoang Vo and Artem Vazhentsev and Artem Shelmanov and Timothy Baldwin and Yova Kementchedjhieva},
      year={2026},
      eprint={2607.04683},
      archivePrefix={arXiv},
      primaryClass={cs.CV},
      url={https://arxiv.org/abs/2607.04683}, 
}

@inproceedings{gpro,
    title = "Addressing Overthinking in Large Vision-Language Models via Gated Perception-Reasoning Optimization",
    author = "Diao, Xingjian  and
      Liu, Zheyuan  and
      Zhang, Chunhui  and
      Wu, Weiyi  and
      Kong, Keyi  and
      Shi, Lin  and
      Ding, Kaize  and
      Vosoughi, Soroush  and
      Gui, Jiang",
    editor = "Liakata, Maria  and
      Moreira, Viviane P.  and
      Zhang, Jiajun  and
      Jurgens, David",
    booktitle = "Findings of the {A}ssociation for {C}omputational {L}inguistics: {ACL} 2026",
    month = jul,
    year = "2026",
    address = "San Diego, California, United States",
    publisher = "Association for Computational Linguistics",
    url = "https://aclanthology.org/2026.findings-acl.215/",
    doi = "10.18653/v1/2026.findings-acl.215",
    pages = "4393--4410",
    ISBN = "979-8-89176-395-1"
}

@misc{avis,
      title={{AVIS}: Adaptive Test-Time Scaling for Vision-Language Models}, 
      author={Ahmadreza Jeddi and Minh Ngoc Le and Amirhossein Kazerouni and Hakki Can Karaimer and Hue Nguyen and Iqbal Mohomed and Michael Brudno and Alex Levinshtein and Konstantinos G. Derpanis and Babak Taati and Radek Grzeszczuk},
      year={2026},
      eprint={2606.11576},
      archivePrefix={arXiv},
      primaryClass={cs.CV},
      url={https://arxiv.org/abs/2606.11576}, 
}

@inproceedings{damani,
 author = {Damani, Mehul and Shenfeld, Idan and Peng, Andi and Bobu, Andreea and Andreas, Jacob},
 booktitle = {International Conference on Learning Representations},
 editor = {Y. Yue and A. Garg and N. Peng and F. Sha and R. Yu},
 pages = {102783--102802},
 title = {Learning How Hard to Think: Input-Adaptive Allocation of {LM} Computation},
 url = {https://proceedings.iclr.cc/paper_files/paper/2025/file/ff414825df833edb8b1839e3d5d495e9-Paper-Conference.pdf},
 volume = {2025},
 year = {2025}
}

@misc{imagine,
      title={When and How Much to Imagine: Adaptive Test-Time Scaling with World Models for Visual Spatial Reasoning}, 
      author={Shoubin Yu and Yue Zhang and Zun Wang and Jaehong Yoon and Huaxiu Yao and Mingyu Ding and Mohit Bansal},
      year={2026},
      eprint={2602.08236},
      archivePrefix={arXiv},
      primaryClass={cs.CV},
      url={https://arxiv.org/abs/2602.08236}, 
}

@InProceedings{ttsvlm,
author="Sammani, Fawaz
and Chamiti, Tzoulio
and Deligiannis, Nikos",
editor="Favaro, Paolo
and Kukelova, Zuzana
and Maki, Atsuto
and Rohrbach, Anna
and Schindler, Konrad
and Tombari, Federico",
title="On Test-Time Scaling for Vision-Language Models",
booktitle="Computer Vision -- ECCV 2026",
year="2026",
publisher="Springer Nature Switzerland",
address="Cham",
pages="168--185",
isbn="978-3-032-37271-0"
}

@inproceedings{visualswap,
title={Are {VLMs} Seeing or Just Saying? {U}ncovering the Illusion of Visual Re-examination},
author={Chufan Shi and Cheng Yang and Yaokang Wu and Linghao Jin and Bo Shui and Taylor Berg-Kirkpatrick and Xuezhe Ma},
booktitle={Forty-third International Conference on Machine Learning},
year={2026},
url={https://openreview.net/forum?id=DdU1o2ZvWi}
}

@misc{sieve,
      title={Improving Visual Reasoning with Iterative Evidence Refinement}, 
      author={Zeru Shi and Kai Mei and Yihao Quan and Dimitris N. Metaxas and Ruixiang Tang},
      year={2026},
      eprint={2603.14117},
      archivePrefix={arXiv},
      primaryClass={cs.CV},
      url={https://arxiv.org/abs/2603.14117}, 
}

@misc{lookagain,
      title={Qwen Look Again: Guiding Vision-Language Reasoning Models to Re-attention Visual Information}, 
      author={Xu Chu and Xinrong Chen and Guanyu Wang and Zhijie Tan and Kui Huang and Wenyu Lv and Tong Mo and Weiping Li},
      year={2025},
      eprint={2505.23558},
      archivePrefix={arXiv},
      primaryClass={cs.CV},
      url={https://arxiv.org/abs/2505.23558}, 
}

@inproceedings{seeingnotbelieving,
 author = {Liu, Zhining and Chen, Ziyi and Liu, Hui and Luo, Chen and Tang, Xianfeng and Wang, Suhang and Zeng, Jingying and Dai, Zhenwei and Shi, Zhan and Wei, Tianxin and Lu, Hanqing and Dumoulin, Benoit and Tong, Hanghang},
 booktitle = {International Conference on Learning Representations},
 editor = {C. Vondrick and B. Hariharan and C. Raffel and L. Pinto and D. Yang and A. Faust},
 pages = {72849--72876},
 title = {Seeing but Not Believing: Probing the Disconnect Between Visual Attention and Answer Correctness in {VLMs}},
 url = {https://proceedings.iclr.cc/paper_files/paper/2026/file/76818d8d85e05e45ce3a16a8468619d1-Paper-Conference.pdf},
 volume = {2026},
 year = {2026}
}

@inproceedings{lostinembeddings,
    title = "Lost in Embeddings: Information Loss in Vision{--}Language Models",
    author = "Li, Wenyan  and
      Tang, Raphael  and
      Li, Chengzu  and
      Zhang, Caiqi  and
      Vuli{\'c}, Ivan  and
      S{\o}gaard, Anders",
    editor = "Christodoulopoulos, Christos  and
      Chakraborty, Tanmoy  and
      Rose, Carolyn  and
      Peng, Violet",
    booktitle = "Findings of the Association for Computational Linguistics: EMNLP 2025",
    month = nov,
    year = "2025",
    address = "Suzhou, China",
    publisher = "Association for Computational Linguistics",
    url = "https://aclanthology.org/2025.findings-emnlp.1235/",
    doi = "10.18653/v1/2025.findings-emnlp.1235",
    pages = "22676--22693",
    ISBN = "979-8-89176-335-7"
}

@misc{equivstructures,
      title={Intriguing Equivalence Structures of the Embedding Space of Vision Transformers}, 
      author={Shaeke Salman and Md Montasir Bin Shams and Xiuwen Liu},
      year={2024},
      eprint={2401.15568},
      archivePrefix={arXiv},
      primaryClass={cs.CV},
      url={https://arxiv.org/abs/2401.15568}, 
}

@InProceedings{mmvp,
    author    = {Tong, Shengbang and Liu, Zhuang and Zhai, Yuexiang and Ma, Yi and LeCun, Yann and Xie, Saining},
    title     = {Eyes Wide Shut? {E}xploring the Visual Shortcomings of Multimodal {LLMs}},
    booktitle = {Proceedings of the IEEE/CVF Conference on Computer Vision and Pattern Recognition (CVPR)},
    month     = {June},
    year      = {2024},
    pages     = {9568-9578}
}

@InProceedings{blindtest,
    author    = {Rahmanzadehgervi, Pooyan and Bolton, Logan and Taesiri, Mohammad Reza and Nguyen, Anh Totti},
    title     = {Vision language models are blind},
    booktitle = {Proceedings of the Asian Conference on Computer Vision (ACCV)},
    month     = {December},
    year      = {2024},
    pages     = {18-34}
}

@inproceedings{scalingattack,
author = {Qixue Xiao and Yufei Chen and Chao Shen and Yu Chen and Kang Li},
title = {Seeing is Not Believing: Camouflage Attacks on Image Scaling Algorithms},
booktitle = {Proceedings of the 28th USENIX Security Symposium},
year = {2019},
isbn = {978-1-939133-06-9},
address = {Santa Clara, CA, USA},
pages = {443--460},
url = {https://www.usenix.org/conference/usenixsecurity19/presentation/xiao},
publisher = {USENIX Association},
month = aug
}

@inproceedings{advpreproc,
author = {Erwin Quiring and David Klein and Daniel Arp and Martin Johns and Konrad Rieck},
title = {Adversarial Preprocessing: Understanding and Preventing Image-Scaling Attacks in Machine Learning},
booktitle = {Proceedings of the 29th USENIX Security Symposium},
year = {2020},
isbn = {978-1-939133-17-5},
pages = {1363--1380},
url = {https://www.usenix.org/conference/usenixsecurity20/presentation/quiring},
publisher = {USENIX Association},
month = aug
}

@inproceedings{randomization,
title={Mitigating Adversarial Effects Through Randomization},
author={Cihang Xie and Jianyu Wang and Zhishuai Zhang and Zhou Ren and Alan Yuille},
booktitle={International Conference on Learning Representations},
year={2018},
url={https://openreview.net/forum?id=Sk9yuql0Z}
}

@InProceedings{perceval,
    author    = {Min, Yingqian and Zhou, Kun and Li, Yifan and Wu, Yuhuan and Peng, Han and Du, Yifan and Zhao, Wayne Xin and Yang, Min and Wen, Ji-Rong},
    title     = {Improving Vision-language Models with Perception-centric Process Reward Models},
    booktitle = {Proceedings of the IEEE/CVF Conference on Computer Vision and Pattern Recognition (CVPR)},
    month     = {June},
    year      = {2026},
    pages     = {33099-33109}
}

@inproceedings{visionthink,
 author = {Yang, Senqiao and Li, Junyi and Lai, Xin and Wu, Jinming and Li, Wei and Ma, Zejun and Yu, Bei and Zhao, Hengshuang and Jia, Jiaya},
 booktitle = {Advances in Neural Information Processing Systems},
 doi = {10.52202/085713-3182},
 editor = {D. Belgrave and C. Zhang and H. Lin and R. Pascanu and P. Koniusz and M. Ghassemi and N. Chen},
 pages = {95187--95227},
 publisher = {Curran Associates, Inc.},
 title = {{VisionThink}: Smart and Efficient Vision Language Model via Reinforcement Learning},
 url = {https://proceedings.neurips.cc/paper_files/paper/2025/file/88be023075a5a3ff3dc3b5d26623fa22-Paper-Conference.pdf},
 volume = {38, Main Conference},
 year = {2025}
}

@inproceedings{badseeing,
title={Bad Seeing or Bad Thinking? {R}ewarding Perception for Multimodal Reasoning},
author={Haozhe Wang and Qixin Xu and Changpeng Wang and Taofeng Xue and Chong Peng and Wenhu Chen and Fangzhen Lin},
booktitle={Forty-third International Conference on Machine Learning},
year={2026},
url={https://openreview.net/forum?id=duzdftKtUA}
}

@inproceedings{seeing2thinking,
title={From Seeing to Thinking: Decoupling Perception and Reasoning Improves Post-Training of Vision-Language Models},
author={Juncheng Wu and Hardy Chen and Haoqin Tu and Xianfeng Tang and Freda Shi and Hui Liu and Hanqing Lu and Cihang Xie and Yuyin Zhou},
booktitle={Forty-third International Conference on Machine Learning},
year={2026},
url={https://openreview.net/forum?id=r7uOjvZdzO}
}

@misc{disentangle,
      title={Disentangling Perception and Reasoning in Multimodal {LLMs} via Reward Design}, 
      author={Omar Sharif and Eftekhar Hossain and Nikhil Singh and Patrick Ng},
      year={2026},
      eprint={2601.00215},
      archivePrefix={arXiv},
      primaryClass={cs.CV},
      url={https://arxiv.org/abs/2601.00215}, 
}

@InProceedings{deeperthought,
    author    = {Peng, Ruiying and Wu, Xueyu and Lei, Jing and Hou, Lu and Ma, Yuanzheng and Li, Xiao-Hui},
    title     = {Deeper Thought, Weaker Aim: Understanding and Mitigating Perceptual Impairment during Reasoning in Multimodal Large Language Models},
    booktitle = {Proceedings of the IEEE/CVF Conference on Computer Vision and Pattern Recognition (CVPR)},
    month     = {June},
    year      = {2026},
    pages     = {12064-12073}
}

@inproceedings{vlmneedwords,
title={{VLMs} Need Words: Vision Language Models Ignore Visual Detail In Favor of Semantic Anchors},
author={Haz Sameen Shahgir and Xiaofu Chen and Yu Fu and Erfan Shayegani and Nael Abu-Ghazaleh and Yova Kementchedjhieva and Yue Dong},
booktitle={Third Conference on Language Modeling},
year={2026},
url={https://openreview.net/forum?id=6DJpzvjWbK}
}

@misc{visualaccess,
      title={Visual Access Boundaries in Vision-Language Model Reasoning}, 
      author={Hiroto Osaka and Shohei Taniguchi and Gouki Minegishi and Kai Yamashita and Masahiro Suzuki and Yutaka Matsuo},
      year={2026},
      eprint={2607.12815},
      archivePrefix={arXiv},
      primaryClass={cs.AI},
      url={https://arxiv.org/abs/2607.12815}, 
}

@inproceedings{compensate,
    title = "{LLM}s Can Compensate for Deficiencies in Visual Representations",
    author = "Takishita, Sho  and
      Gala, Jay  and
      Mohamed, Abdelrahman  and
      Inui, Kentaro  and
      Kementchedjhieva, Yova",
    editor = "Christodoulopoulos, Christos  and
      Chakraborty, Tanmoy  and
      Rose, Carolyn  and
      Peng, Violet",
    booktitle = "Findings of the Association for Computational Linguistics: EMNLP 2025",
    month = nov,
    year = "2025",
    address = "Suzhou, China",
    publisher = "Association for Computational Linguistics",
    url = "https://aclanthology.org/2025.findings-emnlp.825/",
    doi = "10.18653/v1/2025.findings-emnlp.825",
    pages = "15253--15272",
    ISBN = "979-8-89176-335-7"
}

@InProceedings{pearl,
    author    = {Zhang, Chi and Qiu, Haibo and Zhang, Qiming and Xu, Yufei and Zeng, Zhixiong and Yang, Siqi and Shi, Peng and Ma, Lin and Zhang, Jing},
    title     = {Perceptual-Evidence Anchored Reinforced Learning for Multimodal Reasoning},
    booktitle = {Proceedings of the IEEE/CVF Conference on Computer Vision and Pattern Recognition (CVPR)},
    month     = {June},
    year      = {2026},
    pages     = {41111-41120}
}

@misc{vlprm,
      title={Training Vision-Language Process Reward Models for Test-Time Scaling in Multimodal Reasoning: Key Insights and Lessons Learned}, 
      author={Brandon Ong and Tej Deep Pala and Vernon Toh and William Chandra Tjhi and Soujanya Poria},
      year={2025},
      eprint={2509.23250},
      archivePrefix={arXiv},
      primaryClass={cs.AI},
      url={https://arxiv.org/abs/2509.23250}, 
}

@misc{tikart,
      title={{TikArt}: Stabilizing Aperture-Guided Fine-Grained Visual Reasoning with Reinforcement Learning}, 
      author={Hao Ding and Zhichuan Yang and Weijie Ge and Ziqin Gao and Chaoyi Lu and Lei Zhao},
      year={2026},
      eprint={2602.14482},
      archivePrefix={arXiv},
      primaryClass={cs.CV},
      url={https://arxiv.org/abs/2602.14482}, 
}

@InProceedings{vstar,
    author    = {Wu, Penghao and Xie, Saining},
    title     = {{V$^*$}: Guided Visual Search as a Core Mechanism in Multimodal {LLMs}},
    booktitle = {Proceedings of the IEEE/CVF Conference on Computer Vision and Pattern Recognition (CVPR)},
    month     = {June},
    year      = {2024},
    pages     = {13084-13094}
}

@misc{qwen25vl,
      title={{Qwen2.5-VL} Technical Report}, 
      author={Shuai Bai and Keqin Chen and Xuejing Liu and Jialin Wang and Wenbin Ge and Sibo Song and Kai Dang and Peng Wang and Shijie Wang and Jun Tang and Humen Zhong and Yuanzhi Zhu and Mingkun Yang and Zhaohai Li and Jianqiang Wan and Pengfei Wang and Wei Ding and Zheren Fu and Yiheng Xu and Jiabo Ye and Xi Zhang and Tianbao Xie and Zesen Cheng and Hang Zhang and Zhibo Yang and Haiyang Xu and Junyang Lin},
      year={2025},
      eprint={2502.13923},
      archivePrefix={arXiv},
      primaryClass={cs.CV},
      url={https://arxiv.org/abs/2502.13923}, 
}

@InProceedings{llava15,
    author    = {Liu, Haotian and Li, Chunyuan and Li, Yuheng and Lee, Yong Jae},
    title     = {Improved Baselines with Visual Instruction Tuning},
    booktitle = {Proceedings of the IEEE/CVF Conference on Computer Vision and Pattern Recognition (CVPR)},
    month     = {June},
    year      = {2024},
    pages     = {26296-26306}
}

@inproceedings{vicrop,
 author = {Zhang, Jiarui and Khayatkhoei, Mahyar and Chhikara, Prateek and Ilievski, Filip},
 booktitle = {International Conference on Learning Representations},
 editor = {Y. Yue and A. Garg and N. Peng and F. Sha and R. Yu},
 pages = {68194--68213},
 title = {{MLLMs} Know Where to Look: Training-free Perception of Small Visual Details with Multimodal {LLMs}},
 url = {https://proceedings.iclr.cc/paper_files/paper/2025/file/aaa0ac4253da75faf9b0dc0dda062612-Paper-Conference.pdf},
 volume = {2025},
 year = {2025}
}

@misc{chameleon,
      title={Chameleon: Adaptive Adversarial Agents for Scaling-Based Visual Prompt Injection in Multimodal {AI} Systems}, 
      author={M Zeeshan and Saud Satti},
      year={2025},
      eprint={2512.04895},
      archivePrefix={arXiv},
      primaryClass={cs.AI},
      url={https://arxiv.org/abs/2512.04895}, 
}

@inproceedings{patchsize,
title={Steal the Patch Size: Adversarially Manipulate Vision Language Models},
author={Kai Hu and Akash Bharadwaj and Weichen Yu and Matt Fredrikson},
booktitle={Forty-third International Conference on Machine Learning},
year={2026},
url={https://openreview.net/forum?id=LrMoEItzRJ}
}

@inproceedings{controltasks,
    title = "Designing and Interpreting Probes with Control Tasks",
    author = "Hewitt, John  and
      Liang, Percy",
    editor = "Inui, Kentaro  and
      Jiang, Jing  and
      Ng, Vincent  and
      Wan, Xiaojun",
    booktitle = "Proceedings of the 2019 Conference on Empirical Methods in Natural Language Processing and the 9th International Joint Conference on Natural Language Processing (EMNLP-IJCNLP)",
    month = nov,
    year = "2019",
    address = "Hong Kong, China",
    publisher = "Association for Computational Linguistics",
    url = "https://aclanthology.org/D19-1275/",
    doi = "10.18653/v1/D19-1275",
    pages = "2733--2743"
}

@misc{bayessuff,
      title={{Bayes}-Sufficient Representations in Supervised Learning}, 
      author={Vasileios Sevetlidis},
      year={2026},
      eprint={2606.04045},
      archivePrefix={arXiv},
      primaryClass={cs.LG},
      url={https://arxiv.org/abs/2606.04045}, 
}

@misc{fibercriterion,
      title={A Fiber Criterion for Representation Identifiability in Supervised Learning}, 
      author={Vasileios Sevetlidis},
      year={2026},
      eprint={2606.01092},
      archivePrefix={arXiv},
      primaryClass={cs.LG},
      url={https://arxiv.org/abs/2606.01092}, 
}

@INPROCEEDINGS{unaligning,
  author={Salman, Shaeke and Shams, Montasir and Nishino, Mao and Liu, Xiuwen},
  booktitle={2025 International Joint Conference on Neural Networks (IJCNN)}, 
  title={Unaligning Everything: Or Aligning Any Text to Any Image in Multimodal Models}, 
  year={2025},
  volume={},
  number={},
  pages={1-9},
  doi={10.1109/IJCNN64981.2025.11228287}}

@INPROCEEDINGS{badchars,
  author={Boucher, Nicholas and Shumailov, Ilia and Anderson, Ross and Papernot, Nicolas},
  booktitle={2022 IEEE Symposium on Security and Privacy (SP)}, 
  title={Bad Characters: Imperceptible {NLP} Attacks}, 
  year={2022},
  volume={},
  number={},
  pages={1987-2004},
  doi={10.1109/SP46214.2022.9833641}}

@Article{lll,
author="Lenstra, A. K.
and Lenstra, H. W.
and Lov{\'a}sz, L.",
title="Factoring polynomials with rational coefficients",
journal="Mathematische Annalen",
year="1982",
month="Dec",
day="01",
volume="261",
number="4",
pages="515--534",
issn="1432-1807",
doi="10.1007/BF01457454",
url="https://doi.org/10.1007/BF01457454"
}

@inproceedings{navit,
 author = {Dehghani, Mostafa and Mustafa, Basil and Djolonga, Josip and Heek, Jonathan and Minderer, Matthias and Caron, Mathilde and Steiner, Andreas and Puigcerver, Joan and Geirhos, Robert and Alabdulmohsin, Ibrahim M and Oliver, Avital and Padlewski, Piotr and Gritsenko, Alexey and Lucic, Mario and Houlsby, Neil},
 booktitle = {Advances in Neural Information Processing Systems},
 doi = {10.52202/075280-0106},
 editor = {A. Oh and T. Naumann and A. Globerson and K. Saenko and M. Hardt and S. Levine},
 pages = {2252--2274},
 publisher = {Curran Associates, Inc.},
 title = {Patch n’ Pack: {NaViT}, a Vision Transformer for any Aspect Ratio and Resolution},
 url = {https://proceedings.neurips.cc/paper_files/paper/2023/file/06ea400b9b7cfce6428ec27a371632eb-Paper-Conference.pdf},
 volume = {36},
 year = {2023}
}

@misc{qwen2vl,
      title={{Qwen2-VL}: Enhancing Vision-Language Model's Perception of the World at Any Resolution}, 
      author={Peng Wang and Shuai Bai and Sinan Tan and Shijie Wang and Zhihao Fan and Jinze Bai and Keqin Chen and Xuejing Liu and Jialin Wang and Wenbin Ge and Yang Fan and Kai Dang and Mengfei Du and Xuancheng Ren and Rui Men and Dayiheng Liu and Chang Zhou and Jingren Zhou and Junyang Lin},
      year={2024},
      eprint={2409.12191},
      archivePrefix={arXiv},
      primaryClass={cs.CV},
      url={https://arxiv.org/abs/2409.12191}, 
}

@Article{internvl15,
author="Chen, Zhe
and Wang, Weiyun
and Tian, Hao
and Ye, Shenglong
and Gao, Zhangwei
and Cui, Erfei
and Tong, Wenwen
and Hu, Kongzhi
and Luo, Jiapeng
and Ma, Zheng
and Ma, Ji
and Wang, Jiaqi
and Dong, Xiaoyi
and Yan, Hang
and Guo, Hewei
and He, Conghui
and Shi, Botian
and Jin, Zhenjiang
and Xu, Chao
and Wang, Bin
and Wei, Xingjian
and Li, Wei
and Zhang, Wenjian
and Zhang, Bo
and Cai, Pinlong
and Wen, Licheng
and Yan, Xiangchao
and Dou, Min
and Lu, Lewei
and Zhu, Xizhou
and Lu, Tong
and Lin, Dahua
and Qiao, Yu
and Dai, Jifeng
and Wang, Wenhai",
title="How far are we to {GPT-4V}? {C}losing the gap to commercial multimodal models with open-source suites",
journal="Science China Information Sciences",
year="2024",
month="Dec",
day="13",
volume="67",
number="12",
pages="220101",
issn="1869-1919",
doi="10.1007/s11432-024-4231-5",
url="https://doi.org/10.1007/s11432-024-4231-5"
}

@article{
llavaonevision,
title={{LLaVA-OneVision}: Easy Visual Task Transfer},
author={Bo Li and Yuanhan Zhang and Dong Guo and Renrui Zhang and Feng Li and Hao Zhang and Kaichen Zhang and Peiyuan Zhang and Yanwei Li and Ziwei Liu and Chunyuan Li},
journal={Transactions on Machine Learning Research},
issn={2835-8856},
year={2025},
url={https://openreview.net/forum?id=zKv8qULV6n},
note={}
}

@book{lecam1986, title={Asymptotic Methods in Statistical Decision Theory}, ISBN={9781461249467}, ISSN={0172-7397}, url={http://dx.doi.org/10.1007/978-1-4612-4946-7}, DOI={10.1007/978-1-4612-4946-7}, series={Springer Series in Statistics}, publisher={Springer New York}, author={Le Cam, Lucien}, year={1986} }

@book{tsybakov2009, title={Introduction to Nonparametric Estimation}, ISBN={9780387790527}, ISSN={0172-7397}, url={http://dx.doi.org/10.1007/b13794}, DOI={10.1007/b13794}, series={Springer Series in Statistics}, publisher={Springer New York}, author={Tsybakov, Alexandre B.}, year={2009} }

@article{blackwell1953, title={Equivalent Comparisons of Experiments}, volume={24}, ISSN={0003-4851}, url={http://dx.doi.org/10.1214/aoms/1177729032}, DOI={10.1214/aoms/1177729032}, number={2}, journal={The Annals of Mathematical Statistics}, publisher={Institute of Mathematical Statistics}, author={Blackwell, David}, year={1953}, month=jun, pages={265--272} }

@ARTICLE{howard1966,
  author={Howard, Ronald A.},
  journal={IEEE Transactions on Systems Science and Cybernetics}, 
  title={Information Value Theory}, 
  year={1966},
  volume={2},
  number={1},
  pages={22-26},
  doi={10.1109/TSSC.1966.300074}}

@book{vaidyanathan1993,
author = {Vaidyanathan, P. P.},
title = {Multirate Systems and Filter Banks},
publisher = {Prentice Hall},
address = {Englewood Cliffs, NJ},
year = {1993},
isbn = {978-0-13-605718-5}
}
